\documentclass[pdflatex,sn-mathphys-num]{sn-jnl}

\usepackage{comment}
\usepackage[T1]{fontenc}
\usepackage{xcolor}
\usepackage{amsmath}
\usepackage{amssymb}
\usepackage{amsfonts}
\usepackage{amsthm}
\newtheorem{proposition}{Proposition}
\newtheorem{lemma}[proposition]{Lemma}
\usepackage{graphicx}
\usepackage{booktabs}
\usepackage{multirow}
\usepackage{makecell}    
\usepackage{xurl}        
\let\orcidlogo\relax     
\usepackage{orcidlink}   
\hypersetup{hidelinks,bookmarksdepth=3}   
\makeatletter
\renewcommand\scriptsize{\@setfontsize\scriptsize{8bp}{9.5bp}}
\newif\ifsn@Htab
\renewenvironment{table}[1][]%
{\def\sn@tabpos{#1}\def\sn@tabH{H}%
 \ifx\sn@tabpos\sn@tabH\global\sn@Htabtrue\else\global\sn@Htabfalse\fi
 \ifsn@Htab\par\addvspace{\intextsep}\noindent\begin{minipage}{\linewidth}\def\@captype{table}%
 \else\begin{tableorg}[#1]\fi
 \begin{center}\tablebodyfont
 \let\sn@apbr\@arrayparboxrestore\def\@arrayparboxrestore{\sn@apbr\rightskip\z@\@plus 2em\relax\hyphenpenalty 5000\relax}}%
{\end{center}%
 \ifsn@Htab\end{minipage}\par\addvspace{\intextsep}\else\end{tableorg}\fi}
\def\ps@headings{\def\@oddhead{}\def\@evenhead{}\def\@oddfoot{\hfill\thepage\hfill}\let\@evenfoot\@oddfoot\let\@mkboth\markboth}
\makeatother
\begin{document}

\title[Extraction meets Distillation]{\bf Extraction meets Distillation \linebreak -- cryptanalytic techniques to recover parameters from a Deep Neural Network\footnote{This paper significantly expands the conference paper accepted for CSCML 2026. The last-named author joined the collaboration after the conference version's acceptance, contributing to the mathematical analysis and theoretical development of the current version. This is presently a working draft towards a possible journal submission.}}

\author*[1]{\fnm{Biprarshi} \sur{Biswas}\,\orcidlink{0009-0003-4701-8427}}\email{biprarshib@gmail.com}
\author*[2]{\fnm{Arya Anshuman} \sur{Ojha}}\email{kalingasandhababun@gmail.com}
\author*[1]{\fnm{Yubaraj} \sur{Das}}\email{yyubarajddas9876@gmail.com}
\author*[3]{\fnm{Devansh} \sur{Khosla}}\email{devanshkhosla289@gmail.com}
\author*[3]{\fnm{Kamran} \sur{Rais}}\email{kamranrais41@gmail.com}
\author*[4]{\fnm{Saankhya} \sur{Samanta}}\email{saankhya.samanta99@gmail.com}
\author*[5]{\fnm{Subhamoy} \sur{Maitra}}\email{subho@isical.ac.in, maitra.subhamoy@gmail.com}
\author*[6]{\fnm{Pantelimon} \sur{St\u{a}nic\u{a}}}\email{pstanica@nps.edu, pstanica@gmail.com}

\affil*[1]{\orgname{Jadavpur University}, \orgaddress{\city{Kolkata}, \country{India}}}
\affil*[2]{\orgname{National Institute of Science Education and Research}, \orgaddress{\city{Bhubaneswar}, \country{India}}}
\affil*[3]{\orgname{Institute of Mathematics and Application}, \orgaddress{\city{Bhubaneswar}, \country{India}}}
\affil*[4]{\orgname{Indian Institute of Technology}, \orgaddress{\city{Kharagpur}, \country{India}}}
\affil*[5]{\orgname{Indian Statistical Institute}, \orgaddress{\city{Kolkata}, \country{India}}}
\affil*[6]{\orgname{Naval Postgraduate School, Applied Mathematics Department}, \orgaddress{\city{Monterey}, \country{CA 93943, USA}}}

\abstract{We study fully connected neural-network reconstruction by combining cryptanalytic weight-direction extraction with learning from predicted labels. We adapt a ReLU extraction implementation to Leaky ReLU, compare sign-search methods, and estimate biases and output layers. We prove conditional bounds for direction recovery with approximate prefixes, preservation of local distance votes, prediction changes from aligned parameter errors, and sign selection with independent scoring data. We also specify exact row restoration and distinguish its guarantee from that of gradient masking. The experiments evaluate assisted extraction, with privileged access to gradients, hidden activations, and victim parameters, followed by label-based reconstruction. Across eight search seeds on the largest synthetic architecture, Leaky ReLU recovers 51--57 neurons in the deepest hidden layer, while ReLU recovers none. This advantage does not hold uniformly across architectures and layers. Removing victim-weight access from the sign walk gives less accurate signs, despite high final prediction agreement. Across 40 seeded runs on the two smallest synthetic victims, agreement on 5,000 inputs excluded from fitting and selection ranges from $0.9890$ to $0.9992$. On two CIFAR-10 victims, the reconstructed models achieve $56.40\%$ and $58.29\%$ agreement and exceed the distillation baseline on three functional metrics. We report label requests separately from privileged access and explain why prediction agreement does not identify network parameters.}



\maketitle

\section{Introduction}
\label{sec:intro}
A model extraction attack uses access to a trained model to construct another model with similar predictions or parameters. These objectives differ. A model can agree with a classifier on a test set without recovering its weights, and a collection of accurate weight directions need not produce a useful classifier. We study how to assemble partially recovered directions into a complete network while retaining those directions during subsequent training.

Cryptanalytic extraction uses the piecewise-affine structure of networks with rectified linear activations. Carlini et al.~\cite{carlini2020} developed an extraction method using numerical network outputs. Canales-Mart\'inez et al.~\cite{canales2024} improved the recovery procedures, including sign recovery. Chen et al.~\cite{chen2024hardlabel} and Carlini et al.~\cite{carlini2025} studied extraction from predicted labels. RayS and HopSkipJumpAttack~\cite{rays2020,hopskipjump2020} are related decision-based adversarial attacks. They find inputs that change a prediction, rather than recover the model parameters. Liu et al.~\cite{liu2026coverage} studied end-to-end extraction with numerical outputs. Qi et al.~\cite{qi2026leakyrelu} considered several activation functions, including Leaky ReLU. We do not claim that extending extraction beyond ReLU is new by itself.

We distinguish five output interfaces, namely the predicted label, the label and its probability, the top-$k$ probabilities, the full probability vector, and the raw logits. A label from $K$ classes carries at most $\log_2 K$ bits. This is an upper bound on the information in one response, not a parameter-recovery lower bound. Numerical outputs provide additional constraints, but they do not make exact recovery of a deep network an ordinary linear regression problem. Our intended oracle returns only a predicted label.

Our implementation builds on the code of Carlini et al.~\cite{carlini2025}. Several extraction stages read information unavailable through that oracle, including gradients, hidden activations, and true prefix weights. We disclose these accesses in Section~\ref{sec:cheats}. We therefore report the experiments as assisted extraction and reconstruction. The later learning stages use cached victim labels, but this does not remove the privileged information used to construct their initial parameters.

In this paper, we combine three procedures. First, we adapt the local slope calculations to Leaky ReLU. We evaluate accepted directions across recorded search seeds, subject to the ground-truth acceptance procedure described below. The advantage depends on architecture and layer. Second, we compare simulated annealing and parallel tempering for sign search using cached labels and the reconstructed model's logits. Third, we estimate biases from dual points, fit an output layer, and refine the network with recovered-row gradient masks. We state separately how row restoration enforces exact equality. We also batch the dual-point search and stream the clustering data. Our implementation and experimental scripts are listed in~\cite{ourimp}.

Section~\ref{app:prelim} defines the notation and the relevant identification ambiguities. Section~\ref{sec:baseline} describes the extraction stages and their privileged accesses. Sections~\ref{sec:leaky}, \ref{sec:signsearch}, and \ref{sec:mlrecon} describe the Leaky ReLU adaptation, sign search, and reconstruction. Section~\ref{sec:exp} reports the experiments. We compare functional agreement separately from the reported direction coverage, because agreement alone does not establish parameter recovery.

\section{Preliminaries and Notation}
\label{app:prelim}
We write $h_0(x)=x$ and define a network with $r$ hidden layers by
\begin{align*}
 \eta_i(x)=A^{(i)}h_{i-1}(x)+\boldsymbol b^{(i)},\\
 h_i(x)=\sigma(\eta_i(x))\quad(1\le i\le r),\\
 f_\theta(x)=A^{(r+1)}h_r(x)+\boldsymbol b^{(r+1)}.
\end{align*}
Here $A^{(i)}\in\mathbb R^{d_i\times d_{i-1}}$, $\boldsymbol b^{(i)}\in\mathbb R^{d_i}$, and $d_{r+1}=K$ is the number of classes. The activation acts on each coordinate. We write $\|v\|_2$ and $\|v\|_\infty$ for the Euclidean and maximum norms. For matrices, $\|A\|_2$ denotes the induced Euclidean operator norm. We also define $\|A\|_{2\to\infty}=\max_j\|A_j\|_2$, the norm from Euclidean input to maximum output. Here $A_j$ is row $j$ of $A$. We define $\|A\|_F=(\sum_{j,k}A_{jk}^2)^{1/2}$ as the Frobenius norm. We write $\langle u,v\rangle=u^{\mathsf T}v$ for the Euclidean inner product, $V^\perp$ for the orthogonal complement of a subspace $V$, and $\mathbf1[B]$ for the indicator of a statement $B$. A unit vector has Euclidean norm one. The singular values of a matrix with $n$ columns are the $n$ nonnegative square roots of the eigenvalues of $A^{\mathsf T}A$, counted with multiplicity. This convention includes zero values when the matrix has fewer than $n$ rows. The kernel and rank of a matrix are its null space and the dimension of its image, respectively. We write $A_j^{(i)}$ for row $j$ of $A^{(i)}$, and $\eta_{i,j}$ for coordinate $j$ of $\eta_i$. The parameter tuple is $\theta=((A^{(i)},\boldsymbol b^{(i)}))_{i=1}^{r+1}$. The number of hidden neurons is $N=\sum_{i=1}^r d_i$.

We define the hard-label oracle by
\[
 \mathcal O(x)=\arg\max_{1\le k\le K}[f_\theta(x)]_k,
\]
with a fixed rule for ties. Replacing logits $f_k$ by the softmax probabilities $\exp(f_k)/\sum_c\exp(f_c)$ gives the same label. The intended attacker knows the architecture and activation and may query $\mathcal O$, but does not know $\theta$.

We use ReLU and Leaky ReLU~\cite{maas2013leakyrelu}. For $0\le\alpha<1$, we write
\[
 \sigma_\alpha(t)=\begin{cases}t,&t>0,\\ \alpha t,&t\le0.\end{cases}
\]
ReLU is the case $\alpha=0$. For $0<\alpha<1$, this equals $\max(\alpha t,t)$. The restriction $\alpha<1$ is essential for the slope change used by the attack. Both activations are 1-Lipschitz, meaning that $|\sigma_\alpha(u)-\sigma_\alpha(v)|\le|u-v|$. This follows from their piecewise slopes in $[0,1]$, including when the segment crosses zero. Applying the inequality coordinatewise gives the same property in Euclidean norm. In general, a map $g$ is $L$-Lipschitz from one norm to another when $\|g(u)-g(v)\|_{\mathrm{out}}\le L\|u-v\|_{\mathrm{in}}$ for all inputs under consideration.

An activation pattern records which hidden pre-activations are positive. We assign zero pre-activations to the nonpositive state. On a region with fixed pattern $\rho$, define the diagonal matrix $D_i(\rho)$ to have entry $1$ at a positive pre-activation and $\alpha$ otherwise. This convention gives $\sigma_\alpha(\eta_i)=D_i(\rho)\eta_i$ also when a coordinate vanishes. The network is affine on that region. We write it as $M_\rho x+c_\rho$, where
\[
 M_\rho=A^{(r+1)}D_r(\rho)A^{(r)}\cdots D_1(\rho)A^{(1)}.
\]
Empty patterns are omitted. The closures of the nonempty regions cover the input space and may overlap on their boundaries.

We define the critical set of neuron $(i,j)$ by $\mathcal H_{i,j}=\{x:\eta_{i,j}(x)=0\}$. For $i=1$, this is a hyperplane if $A_j^{(1)}\ne0$. For deeper layers, its intersection with a fixed prefix region is the zero set of an affine function restricted to that region. This intersection can be empty, the whole region, or a piece of a hyperplane. In the first layer, we set $\widetilde A_j^{(1)}=A_j^{(1)}$. For $i>1$, the effective normal within a fixed prefix region is
\[
 \widetilde A_j^{(i)}=A_j^{(i)}D_{i-1}A^{(i-1)}\cdots D_1A^{(1)}.
\]
This normal may vanish. The term \emph{critical hyperplane} refers to a nondegenerate local piece of the critical set.

For distinct classes $m,n$, we define
\[
 \mathcal D_{m,n}=\{x:[f_\theta(x)]_m=[f_\theta(x)]_n=\max_k[f_\theta(x)]_k\}.
\]
A regular piece of this set is a decision boundary. It is locally a hyperplane when the corresponding logit difference has a nonzero gradient. A \emph{dual point} for neuron $(i,j)$ lies in $\mathcal H_{i,j}\cap\mathcal D_{m,n}$. For the geometric recovery procedure, we additionally require that only this neuron changes state, that the two classes remain the relevant competitors, and that the intersection is transverse, meaning that its critical normal and decision-boundary normal are linearly independent. We call the local intersection the dual space. Its tangent space consists of the infinitesimal directions parallel to both local hyperplanes, which we use in the rank calculations. Under these assumptions, the dual tangent space has codimension two. The Jacobian of a vector-valued map is its matrix of first partial derivatives on the relevant affine region. A critical crossing need not create a visible decision-boundary bend.

Hard labels do not identify a unique parameter tuple. Hidden-neuron permutations preserve the network function. Positive homogeneity, $\sigma_\alpha(ct)=c\sigma_\alpha(t)$ for $c>0$, gives another ambiguity. Multiplying an incoming weight row and its bias by $c$ and dividing the corresponding outgoing column by $c$ preserves all logits. Moreover, replacing the output layer by
\[
 A^{(r+1)}\longmapsto\lambda A^{(r+1)}+\boldsymbol1 v^{\mathsf T},\qquad
 \boldsymbol b^{(r+1)}\longmapsto\lambda\boldsymbol b^{(r+1)}+\beta\boldsymbol1,
 \qquad \lambda>0,
\]
preserves every label. Here $\boldsymbol1\in\mathbb R^K$, $v\in\mathbb R^{d_r}$, and $\beta\in\mathbb R$. Thus an output-layer fit from labels is a classifier reconstruction, not a unique recovery of the victim's output weights. When we discuss recovered directions, we compare them after matching neurons and allowing a nonzero scalar.

\section{The Extraction Pipeline}
\label{sec:baseline}
We extract hidden-layer directions sequentially. The three stages find dual points, cluster them to estimate directions, and assign signs. Reconstruction then estimates biases and fits the remaining parameters. We separate the geometric procedure from the privileged accesses in its current implementation.

\subsection{Dual Point Discovery}
\label{sec:baseline:phases}
Following Carlini et al.~\cite{carlini2025}, we start with two inputs having different labels and bisect their connecting segment to locate a label transition. We then estimate a local boundary direction and walk along the boundary, reprojecting after each step. A change in its orientation provides a candidate dual point. We store a triplet $(x_{\mathrm{left}},x_{\mathrm{dual}},x_{\mathrm{right}})$ consisting of the candidate and a boundary point on either side of it. The left and right points are not inherently the positive and negative sides of the neuron. That assignment requires sign information.

This interpretation assumes a regular two-class boundary and an isolated critical crossing. A change in the competing classes, simultaneous neuron crossings, or numerical error can also change the estimated orientation. Conversely, a critical crossing with zero downstream effect need not be visible.

We advance multiple boundary walks in batches. The batched implementation retains the search thresholds and step rules. This reduces repeated evaluation overhead, but retaining those constants does not prove that sequential and parallel implementations generate identical samples. Such a comparison would also require matched random streams, stopping rules, and numerical operations. The implementation reads exact gradients and logit gaps at the points described in Section~\ref{sec:cheats}.

\subsection{Signature Recovery}
\label{sec:baseline:stall}
We call a nonzero scalar multiple $cA_j^{(i)}$ of a neuron's weight row its signature. The sign of $c$ is unresolved at this stage. A positive change of scale can be absorbed into the next layer only when the bias and outgoing column are transformed consistently. Negating the incoming row and bias generally changes the network, although degenerate networks can have additional ambiguities.

The geometric recovery method originates in Carlini et al.~\cite[Section~5]{carlini2025}. We give its local linear-algebra formulation to specify exactly what one dual point determines. Write $d=d_0$, and let $x^*$ be an isolated dual point. In a neighborhood of $x^*$, suppose the target critical hyperplane has equation $a^{\mathsf T}(x-x^*)=0$. Suppose also that the same two classes compete on both sides and that their affine logit differences have gradients $n_-$ and $n_+$. The calculation in Section~\ref{sec:leaky:geom} gives $n_+-n_-=\gamma a$ for a scalar $\gamma$.

\begin{proposition}
\label{prop:dual-recovery}
Assume $\gamma\ne0$ and that $a,n_-$ are linearly independent. The tangent space of the dual intersection is
\[
 T=\ker(n_-^{\mathsf T})\cap\ker(n_+^{\mathsf T})
  =\ker(a^{\mathsf T})\cap\ker(n_-^{\mathsf T}).
\]
It has dimension $d-2$. If $T_1,\ldots,T_q$ are tangent spaces from the same first-layer neuron and $U=T_1+\cdots+T_q$ has dimension $d-1$, then $U^\perp=\operatorname{span}(a)$. In particular, two distinct such tangent spaces suffice when $d\ge3$.
\end{proposition}
\begin{proof}
A vector $v$ satisfies $n_-^{\mathsf T}v=n_+^{\mathsf T}v=0$ exactly when it satisfies $n_-^{\mathsf T}v=0$ and $\gamma a^{\mathsf T}v=0$. This proves the equality of the two kernels. The two independent linear equations give dimension $d-2$.

Each $T_\ell$ lies in $\ker(a^{\mathsf T})$, so $U$ lies in this $(d-1)$-dimensional space. Equality of dimensions gives $U=\ker(a^{\mathsf T})$, and taking orthogonal complements proves the claim. For two distinct $(d-2)$-dimensional subspaces of a common $(d-1)$-dimensional space, their sum has dimension greater than $d-2$ and at most $d-1$. It therefore has dimension $d-1$.
\end{proof}

A single dual point supplies a codimension-two tangent space. Its orthogonal complement is two-dimensional and contains both the critical normal and a boundary normal. It therefore does not generally determine the weight direction. In dimension two, all these tangent spaces are zero, so the two-point tangent argument is unavailable. Point locations or other constraints can still provide information. Nor do nonparallel neuron normals alone guarantee that different-neuron clusters have full rank. The rank test requires diversity among the boundary normals as well.

For layer $i>1$, let $J_\ell$ be the Jacobian of $h_{i-1}$ in the prefix region containing dual point $x_\ell$. Transport each tangent space by defining $B_\ell=J_\ell T_\ell\subseteq\mathbb R^{d_{i-1}}$. With $a=A_j^{(i)\mathsf T}$, differentiating the zero pre-activation along the dual intersection gives $a^{\mathsf T}J_\ell v=0$ for every $v\in T_\ell$. Thus $B_\ell\subseteq\ker(a^{\mathsf T})$. If a matrix $M$ has rows that span the $B_\ell$ and
\begin{equation}
\label{eq:direction-rank}
 \operatorname{rank}M=d_{i-1}-1,
\end{equation}
then $\ker M=\operatorname{span}(a)$. This uses transported tangent vectors. In general, pushing a normal through $J_\ell$ does not give the normal in hidden coordinates.

Point locations provide additional affine constraints. If the true prefix is known, write $z_\ell=h_{i-1}(x_\ell)$ and form an augmented matrix with rows $(z_\ell^{\mathsf T},1)$ and $(u^{\mathsf T},0)$ for transported tangent vectors $u\in B_\ell$. Its kernel contains $(A_j^{(i)},b_j^{(i)})^{\mathsf T}$. Rank $d_{i-1}$ determines that affine parameter vector up to scale. Equivalently, after fixing one point $z_1$, we may add $z_\ell-z_1$ to the tangent constraints. These additional constraints can identify a direction even when tangent spaces alone are insufficient. The empirical signature files nevertheless omit bias estimates, which we compute later from the reconstructed prefix.

Numerically, we choose a unit vector minimizing the residual of the transported matrix. The following result makes the effects of tangent error and prefix error explicit. We compare prefix coordinates after fixing neuron permutations and positive scaling conventions. Otherwise a coordinate mismatch can change the meaning of a recovered row.

\begin{proposition}
\label{prop:approximate-prefix}
Let $a\in\mathbb R^n$ be a unit direction, with $n\ge2$. For $1\le\ell\le m$, let $v_\ell\in\mathbb R^{d_0}$ be a unit dual-tangent vector and let $J_\ell\in\mathbb R^{n\times d_0}$ be the exact prefix Jacobian. Form $M$ with rows $(J_\ell v_\ell)^{\mathsf T}$. Suppose $Ma=0$ and $\operatorname{rank}M=n-1$. Write $\delta>0$ for its smallest positive singular value. Assume estimates satisfy
\[
 \|\widehat v_\ell-v_\ell\|_2\le t_\ell,\qquad
 \|\widehat J_\ell-J_\ell\|_2\le p_\ell,\qquad
 \|J_\ell\|_2\le L_\ell.
\]
Form $\widehat M$ with rows $(\widehat J_\ell\widehat v_\ell)^{\mathsf T}$ and define
\[
 \varepsilon=\left[\sum_{\ell=1}^m
 (p_\ell+L_\ell t_\ell+p_\ell t_\ell)^2\right]^{1/2}.
\]
If $\widehat a$ minimizes $\|\widehat M z\|_2$ over unit vectors $z$, then
\begin{equation}
\label{eq:prefix-direction-bound}
 \sqrt{1-\langle\widehat a,a\rangle^2}
 \le\min\{1,2\varepsilon/\delta\}.
\end{equation}
The left side is the sine of the acute angle between $\widehat a$ and the line spanned by $a$. If $\widehat\delta$ is the second-smallest of the $n$ singular values of $\widehat M$ and $\widehat\delta>\varepsilon$, the denominator $\delta$ may be replaced by $\widehat\delta-\varepsilon$.
\end{proposition}
\begin{proof}
We subtract one transported vector by writing
\begin{align*}
 \widehat J_\ell\widehat v_\ell-J_\ell v_\ell
 ={}&(\widehat J_\ell-J_\ell)v_\ell
 +J_\ell(\widehat v_\ell-v_\ell)\\
 &+(\widehat J_\ell-J_\ell)(\widehat v_\ell-v_\ell).
\end{align*}
The triangle inequality bounds its norm by $p_\ell+L_\ell t_\ell+p_\ell t_\ell$. Summing the squared row bounds gives $\|\widehat M-M\|_2\le\|\widehat M-M\|_F\le\varepsilon$. Since $a$ is an admissible unit vector and $Ma=0$, minimality gives $\|\widehat M\widehat a\|_2\le\|\widehat M a\|_2\le\varepsilon$. Another triangle inequality gives $\|M\widehat a\|_2\le2\varepsilon$.

We write $c=\langle\widehat a,a\rangle$ and $z=\widehat a-ca$, so $z\perp a$. Since $\ker M=\operatorname{span}(a)$, the singular-value decomposition gives $\|Mz\|_2\ge\delta\|z\|_2$. Also $M\widehat a=Mz$ and $\|z\|_2=\sqrt{1-c^2}$. These facts prove \textup{\eqref{eq:prefix-direction-bound}}.

For completeness, the singular-value comparison follows from the variational formula
\[
 s_k(B)=\min_{\substack{V\subseteq\mathbb R^n\\\dim V=n-k+1}}
         \max_{\substack{z\in V\\\|z\|_2=1}}\|Bz\|_2,
 \qquad s_1(B)\ge\cdots\ge s_n(B).
\]
The formula follows by expanding in right singular vectors. Every subspace of the stated dimension intersects the span of the first $k$ such vectors, and choosing the span of vectors $k,\ldots,n$ attains the bound. The inequality $\|\widehat Mz\|_2\le\|Mz\|_2+\varepsilon$ therefore gives $s_k(\widehat M)\le s_k(M)+\varepsilon$. Taking $k=n-1$ yields $\delta\ge\widehat\delta-\varepsilon$. Substitution proves the final claim.
\end{proof}

The estimate requires both the exact rank condition and a justified error budget. An observed gap alone does not establish either premise. The prefix estimates must approximate the appropriate local Jacobians, not only the prefix values at the dual points. The minimization also includes null directions of wide matrices, which an economy singular-value decomposition can omit. If the two unit directions are oriented to have nonnegative inner product, their Euclidean distance is at most $\sqrt2$ times the left side of \textup{\eqref{eq:prefix-direction-bound}}. Indeed, for the acute angle $\vartheta$, we have $2(1-\cos\vartheta)\le2\sin^2\vartheta$. This converts a line estimate to a vector estimate after orientation has been resolved. It does not determine a row magnitude or bias.

The reference procedures use rank tests to group compatible dual points~\cite{canales2024,carlini2025}. The number of points needed depends on their transported span. In the evaluated implementation, clustering also uses true hidden activations, and extraction of deeper layers uses the true prefix. These accesses are separate from the linear-algebra statement above.


We read triplet files sequentially and retain at most two thousand points per neuron. This bounds the stored points per cluster, though it can change which constraints are available for recovery. We store each estimated direction separately from its magnitude-corrected row and metadata. The stored direction leaves its sign unresolved for the later search. Magnitude correction and ground-truth comparisons are privileged operations, as described below.

\subsection{Sign Recovery}
\label{sec:sign-recovery}
We use the distance-to-toggle procedure of Carlini et al.~\cite[Section~6]{carlini2025}. The procedure estimates which side of a critical hyperplane tends to reach a subsequent detectable event sooner. We separate the exact local calculation from the statistical assumption used by the vote.

At an isolated target crossing of neuron $(i,j)$, we fix the activation pattern of the preceding layers. Let $J\in\mathbb R^{d_i\times d_0}$ be the Jacobian of the target layer's pre-activations with respect to the input. We write $e_j\in\mathbb R^{d_i}$ for the $j$th standard coordinate vector and define $a=J^{\mathsf T}e_j\in\mathbb R^{d_0}$. Thus $a$ is the target neuron's critical normal in input coordinates. We assume $a\ne0$. For the nonzero decision-boundary normals $n_+$ and $n_-$ on the positive and negative sides, define the orthogonal tangent projectors
\[
 P_\pm=I_{d_0}-\frac{n_\pm n_\pm^{\mathsf T}}{\|n_\pm\|_2^2},
\]
where $I_{d_0}$ is the identity matrix on $\mathbb R^{d_0}$. When $P_+a$ and $P_-a$ are nonzero, we define the unit tangent directions
\[
 v_+=\frac{P_+a}{\|P_+a\|_2},\qquad
 v_-=-\frac{P_-a}{\|P_-a\|_2}.
\]
These signs ensure that the directions enter opposite sides of the target critical hyperplane. Indeed, symmetry and idempotence of each orthogonal projector give $a^{\mathsf T}P_\pm a=\|P_\pm a\|_2^2$. Hence $a^{\mathsf T}v_+=\|P_+a\|_2>0$ and $a^{\mathsf T}v_-=-\|P_-a\|_2<0$. Different projections and intervening activation changes can still give different effects on the other neurons.

For a controlled comparison, suppose $Jv_+=r+\kappa e_j$ and $Jv_-=-r-\kappa e_j$, where $e_j$ is the target coordinate vector, $r_j=0$, and $\kappa>0$. Let $D_{-j}\in\mathbb R^{d_i\times d_i}$ be diagonal, with zero in entry $j$ and the fixed activation slopes in the other entries and set $q=D_{-j}r$. The post-activation velocities are then $q+\kappa e_j$ and $-q-\alpha\kappa e_j$, respectively. Therefore
\begin{equation}
\label{eq:sign-norm}
 \|\dot h_+\|_2^2-\|\dot h_-\|_2^2=(1-\alpha^2)\kappa^2.
\end{equation}
This follows because $q_j=0$, so the cross terms vanish. The identity explains how controlling other coordinates creates a norm difference. The matched-control condition is an assumption and need not hold for boundary-projected walks.

Before any intervening activation change, consider a future pre-activation with $z_k(0)\ne0$ along a unit-speed straight segment. It has the form $z_k(t)=z_k(0)+t\nu_k$. Its first positive zero occurs at the following distance, with $+\infty$ denoting the absence of a positive zero.
\[
 \tau_k=\begin{cases}
 -z_k(0)/\nu_k,&z_k(0)\nu_k<0,\\
 +\infty,&\text{otherwise}.
 \end{cases}
\]
The earliest future toggle on that segment is $\min_k\tau_k$. The rate $\nu_k$ depends on a downstream linear functional of the target-layer velocity. A larger velocity norm alone does not order these hitting times. After an earlier-layer crossing, the local maps must be recomputed and distances accumulated across segments.

We record the distances $d_+$ and $d_-$ on the two hypothesized sides and vote for the side with smaller distance as the positive side. Censored walks and ties require an explicit rule. The following bound states what a majority vote establishes when the required probabilistic assumption holds.

\begin{lemma}
\label{lem:sign-vote}
Let $L\ge1$ be an integer. Suppose $V_1,\ldots,V_L$ are independent Bernoulli variables indicating correct votes, and $\mathbb E V_\ell\ge1/2+\gamma$ for every $\ell$, with $\gamma>0$. Then
\[
 \Pr\left(\sum_{\ell=1}^L V_\ell\le L/2\right)
 \le\exp(-2L\gamma^2).
\]
For $0<\epsilon<1$, taking $L\ge\log(1/\epsilon)/(2\gamma^2)$ suffices for error probability at most $\epsilon$.
\end{lemma}
\begin{proof}
We write $S=\sum_{\ell=1}^L V_\ell$ and $m=\mathbb E S$. The expectation assumption gives $m\ge L(1/2+\gamma)$. On the event $S\le L/2$, we therefore have $S-m\le-L\gamma$. Independence and the bounds $0\le V_\ell\le1$ allow the one-sided Hoeffding inequality~\cite{hoeffding1963}. Hence
\[
 \Pr(S\le L/2)
 \le\Pr(S-m\le-L\gamma)
 \le\exp\!\left(-\frac{2(L\gamma)^2}{\sum_{\ell=1}^L(1-0)^2}\right)
 =\exp(-2L\gamma^2).
\]
This bound includes a tied vote count in the failure event. The requirement $\exp(-2L\gamma^2)\le\epsilon$ is equivalent to $L\ge\log(1/\epsilon)/(2\gamma^2)$, which proves the sample-size assertion.
\end{proof}

The lemma does not prove a positive $\gamma$ for a given neuron. A consistent implementation artifact can generate unanimous incorrect votes. Reused or adaptively selected walks also require justification before this independence-based bound is interpreted as an error guarantee. The exact concentration result is classical~\cite{hoeffding1963}, while the usefulness of the distance asymmetry is an empirical property of the extraction method.

We next give a deterministic condition for preserving a local distance vote under perturbation. This condition concerns prescribed affine continuations. We compare paths with the same arclength parameter and matched downstream coordinates. Applying it to paths with intervening activation changes requires a separate analysis of the later segments.

\begin{proposition}
\label{prop:local-vote-stability}
Let $\mathcal I_+$ and $\mathcal I_-$ be finite nonempty sets of downstream coordinates. On side $s\in\{+,-\}$, write the exact affine continuations as $z_{s,k}(t)=a_{s,k}+t b_{s,k}$ and their approximations as $\widehat z_{s,k}(t)=\widehat a_{s,k}+t\widehat b_{s,k}$ for $t\ge0$. Suppose, for every coordinate on both sides,
\[
 a_{s,k}b_{s,k}<0,\quad |a_{s,k}|\ge m>0,\quad
 |b_{s,k}|\ge\lambda>0,\quad |a_{s,k}|\le A,
\]
and suppose $|\widehat a_{s,k}-a_{s,k}|\le\varepsilon_a<m$ and $|\widehat b_{s,k}-b_{s,k}|\le\varepsilon_b<\lambda$.
Define the first affine hitting distances
$T_s=\min_{k\in\mathcal I_s}(-a_{s,k}/b_{s,k})$ and
$\widehat T_s=\min_{k\in\mathcal I_s}(-\widehat a_{s,k}/\widehat b_{s,k})$.
With
\[
 E=\frac{\varepsilon_a}{\lambda-\varepsilon_b}
   +\frac{A\varepsilon_b}{\lambda(\lambda-\varepsilon_b)},
\]
we have $|\widehat T_s-T_s|\le E$. If $T_- -T_+>2E$, then $\widehat T_+<\widehat T_-$. The reverse statement holds after exchanging the sides.
\end{proposition}
\begin{proof}
Fix a side $s$ and a coordinate $k\in\mathcal I_s$. Since $|a_{s,k}|\ge m>\varepsilon_a$, the perturbation cannot change the sign of the intercept. Likewise, $|b_{s,k}|\ge\lambda>\varepsilon_b$ preserves the sign of the slope. The exact intercept and slope have opposite signs, so the approximate ones do as well. Both affine functions consequently have a unique positive zero. We denote these zeros by $\tau_{s,k}=-a_{s,k}/b_{s,k}$ and $\widehat\tau_{s,k}=-\widehat a_{s,k}/\widehat b_{s,k}$.

The reverse triangle inequality gives $|\widehat b_{s,k}|\ge\lambda-\varepsilon_b>0$. Subtracting the two ratios yields
\begin{align*}
 |\widehat\tau_{s,k}-\tau_{s,k}|
 &\le\frac{|\widehat a_{s,k}-a_{s,k}|}{|\widehat b_{s,k}|}
 +\frac{|a_{s,k}|\,|b_{s,k}-\widehat b_{s,k}|}
 {|b_{s,k}|\,|\widehat b_{s,k}|}\\
 &\le\frac{\varepsilon_a}{\lambda-\varepsilon_b}
 +\frac{A\varepsilon_b}{\lambda(\lambda-\varepsilon_b)}=E.
\end{align*}
Because $\mathcal I_s$ is finite and nonempty, both minima are attained. Choose $k_s$ minimizing the exact distances. Then
$\widehat T_s\le\widehat\tau_{s,k_s}\le\tau_{s,k_s}+E=T_s+E$.
Choose $\widehat k_s$ minimizing the approximate distances. Then
$T_s\le\tau_{s,\widehat k_s}\le\widehat\tau_{s,\widehat k_s}+E=\widehat T_s+E$.
These two inequalities prove $|\widehat T_s-T_s|\le E$ on each side.

If $T_--T_+>2E$, we have
$\widehat T_--\widehat T_+\ge(T_--E)-(T_++E)=T_--T_+-2E>0$.
Thus the approximate distances have the same strict ordering. Exchanging the two sides proves the reverse assertion.
\end{proof}

For a fixed local affine map $z(t)=c+B(h+t v)$, the intercept is $c+Bh$ and the slope is $Bv$. If the approximations satisfy $\|\widehat B-B\|_2\le e_B$, $\|\widehat c-c\|_2\le e_c$, $\|\widehat h-h\|_2\le e_h$ and $\|\widehat v-v\|_2\le e_v$, then valid coordinatewise bounds are
$\varepsilon_a=e_c+e_B\|h\|_2+(\|B\|_2+e_B)e_h$ and
$\varepsilon_b=e_B\|v\|_2+(\|B\|_2+e_B)e_v$.
Indeed, subtract the two intercepts and write the difference as
$(\widehat c-c)+(\widehat B-B)h+\widehat B(\widehat h-h)$.
The slope difference is $(\widehat B-B)v+\widehat B(\widehat v-v)$.
The triangle inequality and $\|\widehat B\|_2\le\|B\|_2+e_B$ give the bounds.

The proposition requires control of the downstream affine maps, the path directions, and a positive distance gap. Matching the target row alone is insufficient. If these continuations describe actual walks only until distance $H$, the hitting distances must be below $H$ and no unmodelled activation crossing may occur before them. Coordinates omitted from $\mathcal I_s$ must be certified not to toggle before the proposed first event. A certificate using the computed gap is also possible. If $\widehat T_- -\widehat T_+>2E$, then $T_+<T_-$ by the same bound. These certificates preserve a distance vote. Identifying that vote with the neuron's true sign still requires the separate sign-signal assumption in Lemma~\ref{lem:sign-vote}.

We must also specify the frame in which a vote is written. Let $w_*$ be a nonzero victim row and suppose $\widehat w=cw_*$ with $c\ne0$, with the bias scaled by the same factor. The victim frame labels the side where the victim pre-activation is positive. The recovered frame labels the side where the recovered pre-activation is positive. If a vote $s_*\in\{-1,1\}$ specifies an orientation in the victim frame, its representation in the recovered frame is $\operatorname{sign}(c)s_*$. This follows because positive $c$ preserves the two sides and negative $c$ exchanges them. The conversion is an algebraic identity. Computing $\operatorname{sign}(c)$ from $\langle\widehat w,w_*\rangle$ uses victim parameters. The orientation conversion makes this computation before the walk and uses its result in selecting directions. We therefore include that access in the assisted procedure, rather than classify it as evaluation only.

For distances, we can certify an ordering without assigning a value to a censored event. Let $\tau_+,\tau_-\in[0,\infty]$ denote distances to the same specified stopping event on two prescribed walks. Suppose recorded intervals $[\ell_+,u_+]$ and $[\ell_-,u_-]$ contain the corresponding distances. If $u_+<\ell_-$, then $\tau_+\le u_+<\ell_-\le\tau_-$. The reverse ordering follows when $u_-<\ell_+$. Otherwise these intervals alone do not certify an ordering. A walk with no event through a verified distance $H$ gives the interval $[H,\infty]$. A numerical failure gives no such lower bound unless the traversal to $H$ was verified. This certificate concerns the event ordering. The relation between that ordering and the neuron's sign remains a separate assumption.

We can obtain a positive event-probability difference from the controlled velocities under an explicit local model. This gives one sufficient mechanism for a sign signal. It does not assume that trained downstream weights follow this model. We write $\Phi(t)=(2\pi)^{-1/2}\int_{-\infty}^t\exp(-u^2/2)\,du$ for the standard normal distribution function.

\begin{proposition}
\label{prop:local-sign-model}
Let $q\in\mathbb R^{d_i}$ satisfy $q_j=0$, let $\kappa>0$, and let $0\le\alpha<1$. Define $v_+=q+\kappa e_j$ and $v_-=-q-\alpha\kappa e_j$. Let $g$ have independent standard normal coordinates, and fix $z_0\ne0$ and $H>0$ independently of $g$. In the affine continuations $z_\pm(t)=z_0+t\langle g,v_\pm\rangle$, let $B_\pm$ indicate a zero at some $t\in(0,H]$. Then
\[
 p_+=\mathbb E B_+=\Phi\!\left(-\frac{|z_0|}{H\sqrt{\|q\|_2^2+\kappa^2}}\right),\qquad
 p_-=\mathbb E B_-=\Phi\!\left(-\frac{|z_0|}{H\sqrt{\|q\|_2^2+\alpha^2\kappa^2}}\right),
\]
where $p_-=0$ if its denominator is zero. We have $\mu=p_+-p_->0$. For an integer $L\ge1$ and $L$ independent draws of $g$, with paired indicators $B_{\ell,+},B_{\ell,-}$,
\[
 \Pr\!\left(\sum_{\ell=1}^L(B_{\ell,+}-B_{\ell,-})\le0\right)
 \le\exp(-L\mu^2/2).
\]
\end{proposition}
\begin{proof}
For a fixed vector $v$, the rate $\langle g,v\rangle$ is normal with mean zero and variance $\|v\|_2^2$. If $z_0>0$, a zero by $H$ occurs exactly when this rate is at most $-z_0/H$. If $z_0<0$, it occurs when the rate is at least $|z_0|/H$. Normal symmetry gives the displayed probability in either case. A zero rate cannot reach zero from $z_0\ne0$, which proves the convention for zero variance.

Since $q_j=0$, the two variances are $\|q\|_2^2+\kappa^2$ and $\|q\|_2^2+\alpha^2\kappa^2$. The first is strictly larger. For the positive threshold $|z_0|/H$, the displayed normal tail is strictly increasing in the positive standard deviation. Also $p_+>0$ when the negative-side variance vanishes. Thus $\mu>0$ in every case.

For each independent draw, define $D_\ell=B_{\ell,+}-B_{\ell,-}$. The two indicators within a pair need not be independent. Their difference satisfies $-1\le D_\ell\le1$ and $\mathbb E D_\ell=p_+-p_-=\mu$. The draws of $g$ make the variables $D_1,\ldots,D_L$ independent. The event $\sum_\ell D_\ell\le0$ implies $\sum_\ell(D_\ell-\mu)\le-L\mu$. Hoeffding's lower-tail inequality~\cite{hoeffding1963} therefore bounds its probability by
\[
 \exp\!\left(-\frac{2(L\mu)^2}{\sum_{\ell=1}^L(1-(-1))^2}\right)
 =\exp(-L\mu^2/2),
\]
which proves the stated bound.
\end{proof}

This proposition concerns the excess of fixed-horizon events under affine continuation. It does not order the two first-event distances for every draw or prove correctness of the implemented distance vote. The assumptions exclude changes in the local map before $H$ and require an independent Gaussian downstream response with the stated controlled velocities. Actual walks can violate each condition. We can test those conditions using logged local rates and intervening events, but should not identify the model's positive $\mu$ with a measured sign advantage without those checks.


The last hidden layer has no later hidden activation to toggle. The reference method can instead use intersections with other class decision boundaries~\cite[Section~6.2]{carlini2025}. Thus the absence of future hidden layers does not make last-layer sign recovery impossible. The first layer also need not lack a sign signal. The weak endpoint results in our implementation concern its chosen walks and measured events.

The modified runner uses \protect\path{along_decision_boundary} in every layer, replacing the \protect\path{perfect_control_along_decision_boundary} configuration at the endpoints. The latter caps a recorded negative-side distance at three times the positive-side distance. A cap $\min\{d_-,3d_+\}$ with $d_+>0$ preserves the strict comparison $d_->d_+$ when that comparison already holds. The cap alone therefore does not prove a universal positive-vote bias. Censoring, unsuccessful walks, and default sign decisions must be analyzed together. The runner also bounds the inner bisection iterations and saves per-neuron signs and vote summaries. We treat low-information sign estimates as initialization for the label-based search, rather than proof that their confidence values are calibrated.

\subsection{Privileged Information in the Implementation}
\label{sec:cheats}
The reference code~\cite{carlini2025} contains operations that depart from the intended hard-label oracle. We retain its D1--D6 identifiers so that the manuscript can be compared with the code. Table~\ref{tab:cheats} separates runtime information unavailable through labels (type A) from direct parameter reads (type W). These operations assist the attack itself. They are distinct from ground-truth comparisons used only to evaluate a completed attack.

\begin{table}[htb]
\centering
\caption{Privileged accesses in the evaluated extraction implementation.}
\label{tab:cheats}
\setlength{\tabcolsep}{4pt}
\footnotesize
\resizebox{0.8\linewidth}{!}{%
\begin{tabular}{p{2.2cm}p{2.2cm}cp{5.6cm}}
\toprule
Operation & Reads & Type & Role in the reported implementation \\
\midrule
D1 -- Boundary detection & Logit gap & A & In the D1 ablation, disabling it reduces accepted directions on the tiny ReLU victim from 142/256 to 24/256. This is an implementation ablation, not an impossibility result. \\
D2 -- Boundary refinement & Exact logit gradients & A & A label-based fallback exists in the reference code. Its use must be specified in a strict-oracle experiment. \\
D3 -- Neuron clustering & Hidden activations & A & The reference code also provides an SVD-based alternative. Its reported higher cost does not make privileged access necessary in principle. \\
D4 -- Prefix initialization & True weight matrices & W & Supplies the prefix used to extract deeper layers. The present experiments do not replace it with a fully reconstructed prefix. \\
D5 -- Signature scaling & True weight vector & W & Corrects scale and checks candidates. In the Leaky ReLU path, ground-truth agreement also determines candidate acceptance. \\
D6 -- Sign-walk parameters & True weights and biases & W & Supplies the prefix used by the sign-recovery walk. The recovered-frame conversion also compares recovered rows with true rows before choosing directions. \\
\bottomrule
\end{tabular}
}
\end{table}

D4 and D6 avoid the error propagation that a reconstructed prefix would introduce. D5 also affects the interpretation of direction coverage, since accepting a candidate by comparison with its true row uses information unavailable to an attacker. Removing these accesses requires a separate end-to-end evaluation. Later fitting from labels does not retrospectively make the extraction stages label-only.

\section{Leaky ReLU Adaptation}
\label{sec:leaky}
We replace the zero negative-side slope by a fixed $\alpha\in(0,1)$. Qi et al.~\cite{qi2026leakyrelu} also study extraction with several activation functions. Our experiments concern the adaptation of this reference implementation and its use in the subsequent reconstruction.

\subsection{Local Geometry}
\label{sec:leaky:geom}
For a fixed scalar pre-activation, ReLU and Leaky ReLU change slope at the same zero. For fixed first-layer weights and biases, their first-layer critical hyperplanes therefore agree. This does not imply that all deeper critical sets or decision boundaries agree. A deeper pre-activation depends on the outputs of earlier activations, which change when $\alpha$ changes. For example, a second-layer pre-activation $\sigma_\alpha(x)+1$ has no zero when $\alpha=0$, whereas for $\alpha>0$ it vanishes at $x=-1/\alpha$.

The local bend calculation nevertheless has the same form. Consider an isolated crossing of neuron $(i,j)$, with the prefix and all other local activation states fixed. On either side, write its pre-activation as $\eta_{i,j}(x)=a^{\mathsf T}x+b$. Let $u\in\mathbb R^K$ be the downstream linear response to its activation. The difference between the two affine logit maps, extended to the same neighborhood, is
\[
 f_+(x)-f_-(x)=(1-\alpha)u(a^{\mathsf T}x+b).
\]
Thus the gradients of the $m$--$n$ logit difference satisfy
\begin{equation*}
 \nabla(f_m-f_n)_+-\nabla(f_m-f_n)_-
 =(1-\alpha)(u_m-u_n)a.
\end{equation*}
This follows by subtracting the slopes $1$ and $\alpha$ of the target activation and composing with the fixed downstream affine map. If $u_m=u_n$, the crossing has no effect on this logit difference. Even when the right-hand side is nonzero, a visible bend requires the two boundary normals to be nonparallel and the same classes to remain maximal. We therefore transfer the local method, not an assertion of identical geometry or guaranteed detectability.

\subsection{Slope Calculations and Candidate Acceptance}
\label{sec:leaky:algebra}
Prefix linearization uses the matrices $D_i(\rho)$ defined in Section~\ref{app:prelim}. Negative coordinates contribute a slope $\alpha$ instead of zero. We apply this substitution consistently in the PyTorch and NumPy activation functions, the prefix linearization, and the sign-walk forward simulations. The coordinate-hit buffer is allocated with the prefix output dimension $d_{i-1}$, rather than the target output dimension $d_i$. This corrects a shape error on nonuniform architectures under either activation. The same substitution changes the downstream response used by the distance-to-toggle procedure. It can alter the strength of the sign statistic, and no universal accuracy bound follows from $\alpha$ being small.

The ReLU implementation checks for zero-hit prefix coordinates and requires singular values $s_{\mathrm{last}}<10^{-4}$ and $s_{\mathrm{penultimate}}>10^{-2}$. These thresholds approximate the one-dimensional-kernel condition in Eq.~\eqref{eq:direction-rank}. A hit uses the finite threshold $10^{-4}$, so zero hits need not mean mathematically zero activation. The ReLU implementation accepts a direction when its dual-space matrix has a numerically isolated smallest singular value. In the Leaky ReLU path we relax this gate and retain candidates that agree sufficiently well with the true neuron. This comparison uses privileged information. Accordingly, the higher Leaky ReLU coverage measures candidates accepted by the assisted implementation. It does not establish the same coverage for an attacker who must select candidates from labels alone.

\subsection{Rank and Inactive Coordinates}
\label{sec:leaky:jacobian}
For a fixed matrix $B$ mapping an input perturbation to pre-activation perturbations, the local activation map gives $DB$. When $0<\alpha<1$, all diagonal entries of $D$ lie in $[\alpha,1]$. Hence
\begin{proposition}
\label{prop:leaky-rank}
Let $B\in\mathbb R^{k\times n}$ and let $D\in\mathbb R^{k\times k}$ be diagonal with entries in $[\alpha,1]$, where $0<\alpha<1$. For every $v\in\mathbb R^n$,
$
 \alpha\|Bv\|_2\le\|DBv\|_2\le\|Bv\|_2.
$
Consequently $\ker(DB)=\ker B$ and $\operatorname{rank}(DB)=\operatorname{rank}B$. If instead $D$ has diagonal entries in $\{0,1\}$, as for ReLU, then $\operatorname{rank}(DB)\le\operatorname{rank}B$, with strict inequality possible.
\end{proposition}
\begin{proof}
Writing $D=\operatorname{diag}(d_1,\ldots,d_k)$ gives
$\|DBv\|_2^2=\sum_\ell d_\ell^2(Bv)_\ell^2$.
Bounding each $d_\ell^2$ between $\alpha^2$ and 1 proves the inequalities. Since $\alpha>0$, $DBv=0$ if and only if $Bv=0$. Rank--nullity then gives equality of ranks. For ReLU, the image of $DB$ is the image under $D$ of the image of $B$, so its dimension cannot exceed $\operatorname{rank}B$. Strict inequality occurs, for example, with $B=I_2$ and $D=\operatorname{diag}(1,0)$, whose ranks are two and one, respectively.
\end{proof}

This proposition concerns one fixed pre-activation matrix. It does not assert equality of ReLU and Leaky ReLU network Jacobians, since changing the activation also changes deeper patterns and maps. Products with later weight matrices can remain rank deficient. Repeated multiplication by slopes as small as $\alpha$ can worsen conditioning. Thus Leaky ReLU avoids rank loss caused solely by this local diagonal activation, but does not guarantee a unique or well-conditioned SVD solution. Both this mechanism and the relaxed candidate gate can affect the reported coverage.

\section{Metaheuristic Sign Search}
\label{sec:signsearch}
We search over the signs of recovered rows while using cached victim labels to score candidate networks. Within one scoring call, candidate evaluations reuse the cached labels. The driver can invoke this procedure repeatedly and request the same labels again, so we count label acquisition at every invocation. It does not guarantee recovery of the victim's signs, because different parameter assignments can agree on a finite set.

\subsection{Sign Assignments}
\label{sec:signsearch:problem}
For $k$ recovered rows in a layer, we write $s\in\{-1,1\}^k$ for a sign assignment and $\widehat f_s$ for the corresponding reconstructed logits. The recovered row for neuron $j$ is $s_j\widehat A_j$. When we negate a row, we also negate its bias estimate, or recompute that estimate from stored dual points. This preserves the zero set of its affine pre-activation in the current prefix coordinates.

Let $E=\{(x_\ell,y_\ell)\}_{\ell=1}^L$ be a cached scoring set with $y_\ell=\mathcal O(x_\ell)$. We define agreement by
\[
 A_E(s)=\frac1L\sum_{\ell=1}^L
 \mathbf1[\arg\max_c[\widehat f_s(x_\ell)]_c=y_\ell].
\]
Exhaustive search evaluates all $2^k$ assignments. It is practical only for modest $k$, and identifies an optimum of this finite-set objective rather than necessarily the true signs.

\subsection{Single-Flip Search}
\label{sec:signsearch:exact}
A greedy search changes one sign at a time and retains changes that improve agreement. The discrete objective can have plateaus and local optima. Several simultaneous flips can improve agreement even when no individual flip does. This motivates methods that sometimes accept a worse score. The controlled benchmark starts from the true 256-hidden-neuron ReLU model with half its signs corrupted. The reported mean sign accuracies are $0.557$ for greedy search, $0.569$ for simulated annealing with agreement, $0.639$ for simulated annealing with the margin, and $0.666$ for parallel tempering with the margin. The first three means use eight seeds, whereas parallel tempering uses seeds 0--5. These numbers evaluate correction from a privileged initialization and are separate from the end-to-end results in Table~\ref{tab:permodel}. The unequal seed sets also preclude treating their mean differences as paired estimates. Per-seed benchmark outputs are needed for a statistical comparison.

\subsection{Agreement and Margin Objectives}
\label{sec:signsearch:meta}
We define the reconstructed model's mean label margin by
\[
 m_E(s)=\frac1L\sum_{\ell=1}^L\left(
 [\widehat f_s(x_\ell)]_{y_\ell}
 -\max_{c\ne y_\ell}[\widehat f_s(x_\ell)]_c\right).
\]
The label $y_\ell$ comes from the victim, but all logits in this expression come from the reconstructed model. Thus computing the margin does not require victim logits. The margin can distinguish assignments with equal agreement. It is a surrogate objective, not a smooth function of discrete signs, and increasing it need not increase agreement. It also depends on logit scale.

We use agreement to select the final saved candidate and a margin-based score to guide exploration. The margin score is $m_E$ without a stated additional normalization. The shared driver first runs greedy search, stores the best state by $A_E$, and restores that state after exploration. Including the initial state among the saved candidates ensures that the returned training-set agreement is no smaller than at entry. This guarantees neither improved held-out agreement nor more correct signs. Search stops when agreement reaches $1-10^{-12}$. We do not equate final selection by agreement with accepting every search move by agreement.

\subsection{Simulated Annealing and Parallel Tempering}
\label{sec:signsearch:threat}
For a chosen exploration score $S$, simulated annealing~\cite{Kirkpatrick1983} proposes a sign change from $s$ to $s'$ and, for symmetric proposals, accepts it with probability
\[
 \min\{1,\exp((S(s')-S(s))/T)\},\qquad T>0.
\]
The default schedule decreases geometrically from $0.02$ to $10^{-4}$ over four sweeps. A decreasing temperature makes worsening moves less likely. At low temperature this approaches greedy search for $S$, not gradient ascent.

Parallel tempering~\cite{Geyer1991} runs chains at several positive temperatures. For chains at $T_a,T_b$ with states $s_a,s_b$, the standard swap acceptance probability for a score being maximized is
\[
 \min\{1,\exp((T_a^{-1}-T_b^{-1})(S(s_b)-S(s_a)))\}.
\]
The parallel-tempering defaults use six replicas on a geometric temperature ladder from $10^{-3}$ to $0.05$. Both methods can move between local optima, but neither guarantees the best assignment within a fixed computation budget. We use simulated annealing as the default and compare parallel tempering on selected victims. The tables report the combined reconstruction, not an isolated proof of correct sign recovery.

\subsection{Selection with Independent Scoring Data}
\label{sec:signsearch:selection}
We separate the statistical selection problem from the optimizer's search. Let $\mathcal F=\{g_s:s\in\mathcal S\}$ be a nonempty finite family of candidate classifiers, where $\mathcal S$ is a finite index set and $M=|\mathcal S|$. We construct the family using fitting data and fix it before observing the scoring sample. An index can encode a sign assignment and an output layer fitted using only the fitting labels. We allow distinct indices to define the same classifier.

Let $X_1,\ldots,X_n$ be independent scoring inputs, with $n\ge1$. They need not have identical distributions. We define the mean target agreement and its empirical estimate by
\[
 p_s=\frac1n\sum_{\ell=1}^n\Pr(g_s(X_\ell)=\mathcal O(X_\ell)),\qquad
 \widehat p_s=\frac1n\sum_{\ell=1}^n\mathbf1[g_s(X_\ell)=\mathcal O(X_\ell)].
\]
For identically distributed inputs, $p_s$ is ordinary population agreement under their common distribution. For fixed-size independent strata, it is the agreement under the mixture weighted by the stratum proportions. If the family is random, we condition on the fitting data and every random choice used to construct it. The scoring inputs must remain independent with the specified distributions after this conditioning.

\begin{proposition}
\label{prop:sign-selection}
Under these assumptions, for $\varepsilon>0$,
\[
 \Pr\!\left(\max_{s\in\mathcal S}|\widehat p_s-p_s|>\varepsilon\right)
 \le2M\exp(-2n\varepsilon^2).
\]
Suppose the returned index $\widehat s$ satisfies $\widehat p_{\widehat s}\ge\max_s\widehat p_s-\rho$ for a specified $\rho\ge0$. For $0<\beta<1$, with probability at least $1-\beta$,
\begin{equation}
\label{eq:selection-regret}
 \max_s p_s-p_{\widehat s}
 \le 2\sqrt{\frac{\log(2M/\beta)}{2n}}+\rho.
\end{equation}
If $M>1$, there is a unique maximizing index $s_*$, and $p_{s_*}-p_s\ge\Delta>0$ for every $s\ne s_*$, then an exact empirical maximizer satisfies
\begin{equation}
\label{eq:selection-gap}
 \Pr(\widehat s\ne s_*)\le(M-1)\exp(-n\Delta^2/2).
\end{equation}
\end{proposition}
\begin{proof}
For each fixed $s$, the agreement indicators are independent and lie in $[0,1]$. The two one-sided Hoeffding bounds~\cite{hoeffding1963} give $\Pr(|\widehat p_s-p_s|>\varepsilon)\le2\exp(-2n\varepsilon^2)$. Summing over the $M$ indices proves the first assertion. No independence between candidates is needed.

On the event that all deviations are at most $\varepsilon$, choose an index $s_*$ maximizing $p_s$. We have
\begin{align*}
 p_{s_*}-p_{\widehat s}
 ={}&(p_{s_*}-\widehat p_{s_*})
 +(\widehat p_{s_*}-\widehat p_{\widehat s})
 +(\widehat p_{\widehat s}-p_{\widehat s})\\
 \le{}&2\varepsilon+\rho.
\end{align*}
Setting $\varepsilon=\sqrt{\log(2M/\beta)/(2n)}$ proves \textup{\eqref{eq:selection-regret}}.

For the gap assertion, define the paired differences
\[
 Z_\ell(s)=\mathbf1[g_{s_*}(X_\ell)=\mathcal O(X_\ell)]
            -\mathbf1[g_s(X_\ell)=\mathcal O(X_\ell)].
\]
They are independent across $\ell$, lie in $[-1,1]$, and their average expectation is at least $\Delta$. Hoeffding's lower-tail inequality therefore gives
$\Pr(\sum_\ell Z_\ell(s)\le0)\le\exp(-n\Delta^2/2)$.
An incorrect empirical selection requires some other index to tie or exceed $s_*$ in empirical agreement. Summing the bound over the $M-1$ competitors proves \textup{\eqref{eq:selection-gap}}, counting ties conservatively as failures.
\end{proof}

Exhaustive search over $k$ binary signs has $M\le2^k$. A sufficient sample count for selection error at most $\beta$ in the gap setting is $n\ge2\log((M-1)/\beta)/\Delta^2$. For $k=8$, $\Delta=0.05$, and $\beta=0.05$, this gives 6,830 inputs after rounding upward. The value of $\Delta$ here is an assumption, not an estimate for the experimental victims. Proposition~\ref{prop:sign-selection} proves sign identification only when the intended assignment belongs to the family and uniquely maximizes the target agreement. Prediction-equivalent assignments have no positive gap.

For simulated annealing or parallel tempering, \textup{\eqref{eq:selection-regret}} requires a justified optimization error $\rho$. The algorithms do not supply that guarantee merely by saving their best visited state. Nor can we apply a fixed-family bound using only the number of visited candidates when those candidates are trained or refined on the scoring data. A fixed family of all possible candidates may still be considered if it is specified independently, but its size must include that entire family. The exhaustive block experiments in Section~\ref{sec:exp:audited} address this distinction.

\section{Learning-Assisted Reconstruction}
\label{sec:mlrecon}
We construct a complete model from the accepted directions. We estimate biases, fit an output layer, search signs, refit the output layer, and train the remaining parameters while masking recovered-row gradients. Exact row restoration is a separate update rule. These stages use stored dual points and cached labels. Their initialization retains the privileged accesses described in Section~\ref{sec:cheats}.

\subsection{Incomplete Parameters and Sign Errors}
\label{sec:mlrecon:propagation}
Directions alone leave signs, biases, unrecovered neurons, and the output layer unresolved. A sign error has an exact effect at the target coordinate because
\begin{equation}
\label{eq:sign-flip-identity}
 \sigma_\alpha(-z)-\sigma_\alpha(z)=-(1+\alpha)z\qquad(z\in\mathbb R).
\end{equation}
For $z\ge0$, the left side is $-\alpha z-z$. For $z<0$, it is $-z-\alpha z$, giving the same expression. The identity assumes that both the incoming row and its bias are negated. Negating only the row is a different change.

\begin{proposition}
\label{prop:sign-error}
Consider two networks with identical parameters except that one incoming row and its bias in hidden layer $i$ are negated. We denote the original and modified logit maps by $f$ and $\widehat f$, respectively. Let $z(x)$ be that neuron's original pre-activation. If $g_i$ is the common map from $h_i$ to the logits and has Lipschitz constant $L_i$ from Euclidean norm to maximum norm, then
\[
 \|\widehat f(x)-f(x)\|_\infty\le (1+\alpha)L_i|z(x)|.
\]
If $f$ has a unique predicted class $y$ and margin
$\mu(x)=f_y(x)-\max_{c\ne y}f_c(x)$, its label is unchanged whenever
\[
 \mu(x)>2(1+\alpha)L_i|z(x)|.
\]
\end{proposition}
\begin{proof}
The networks have the same prefix through layer $i-1$, so their pre-activations at the modified neuron are $z(x)$ and $-z(x)$. All other pre-activations in layer $i$ agree. By Eq.~\eqref{eq:sign-flip-identity},
\[
 \widehat h_i(x)-h_i(x)=-(1+\alpha)z(x)e_j,
\]
where $e_j$ is the coordinate vector of the modified neuron. Its Euclidean norm is $(1+\alpha)|z(x)|$. The parameters after layer $i$ agree, so the same downstream map $g_i$ acts on both hidden vectors. Its Lipschitz property gives
$\|\widehat f(x)-f(x)\|_\infty\le L_i\|\widehat h_i(x)-h_i(x)\|_2=(1+\alpha)L_i|z(x)|$.

Set $\epsilon=(1+\alpha)L_i|z(x)|$. For every class $c\ne y$, the logit bound gives $\widehat f_y(x)\ge f_y(x)-\epsilon$ and $\widehat f_c(x)\le f_c(x)+\epsilon$. Hence
$\widehat f_y(x)-\widehat f_c(x)\ge f_y(x)-f_c(x)-2\epsilon\ge\mu(x)-2\epsilon$.
Under the stated margin condition, this quantity is positive for every competitor. Thus $y$ remains the unique predicted class.
\end{proof}

Since each hidden activation is 1-Lipschitz, we may take
\[
 L_i=\|A^{(r+1)}\|_{2\to\infty}
       \prod_{k=i+1}^{r}\|A^{(k)}\|_2,
\]
where an empty product is 1. To see this, an affine hidden layer changes the Euclidean distance by at most its matrix norm, its bias cancels from the difference, and the activation does not increase that distance. Composing these bounds through the hidden layers and applying the output norm gives the displayed constant. This proves a sufficient condition for label preservation, not an error-rate estimate. A sign error can be harmless when its outgoing response vanishes or the classification margin is large. No assumption about weight-direction recovery forces $z(x)$ to be positive on half the inputs.

The raw reconstruction also has zero biases and a randomly initialized output layer, as well as missing hidden rows. Its low agreement therefore cannot be attributed to signs alone. The stage-by-stage ablation in Section~\ref{sec:exp:ablation} separates these changes empirically.

We can also bound errors arising at several layers without assuming that a single sign is the only change. Consider two networks with the same architecture and activation, with neurons and positive scale conventions aligned. We write $\widehat A^{(i)},\widehat{\boldsymbol b}^{(i)}$ for the reconstructed parameters, $\widehat h_i$ for their hidden representations, and $\widehat f$ for their logit map. For hidden layers define $e_i=\|\widehat A^{(i)}-A^{(i)}\|_2$ and $q_i=\|\widehat{\boldsymbol b}^{(i)}-\boldsymbol b^{(i)}\|_2$.

\begin{proposition}
\label{prop:network-error}
Fix $R\ge0$ and consider inputs with $\|x\|_2\le R$. Define $B_0=R$, $U_0=0$, and, for $1\le i\le r$,
\[
 B_i=\|A^{(i)}\|_2 B_{i-1}+\|\boldsymbol b^{(i)}\|_2,\qquad
 U_i=\|\widehat A^{(i)}\|_2 U_{i-1}+e_iB_{i-1}+q_i.
\]
Then $\|h_i(x)\|_2\le B_i$ and $\|\widehat h_i(x)-h_i(x)\|_2\le U_i$. Define
\[
 V=\|\widehat A^{(r+1)}\|_{2\to\infty}U_r
 +\|\widehat A^{(r+1)}-A^{(r+1)}\|_{2\to\infty}B_r
 +\|\widehat{\boldsymbol b}^{(r+1)}-\boldsymbol b^{(r+1)}\|_\infty.
\]
Then $\|\widehat f(x)-f_\theta(x)\|_\infty\le V$. If the victim has a unique winning class and margin $\mu(x)>2V$, the two classifiers give the same label.
\end{proposition}
\begin{proof}
Fix an input with $\|x\|_2\le R$. At layer zero, $h_0(x)=\widehat h_0(x)=x$, so $\|h_0(x)\|_2\le B_0$ and $\|\widehat h_0(x)-h_0(x)\|_2=U_0=0$.

Suppose the stated bounds hold at layer $i-1$. The activation vanishes at zero and is 1-Lipschitz. Therefore
\begin{align*}
 \|h_i(x)\|_2
 &\le\|A^{(i)}h_{i-1}(x)+\boldsymbol b^{(i)}\|_2\\
 &\le\|A^{(i)}\|_2 B_{i-1}+\|\boldsymbol b^{(i)}\|_2=B_i.
\end{align*}
For the reconstructed pre-activation, write $\widehat\eta_i(x)=\widehat A^{(i)}\widehat h_{i-1}(x)+\widehat{\boldsymbol b}^{(i)}$. Subtract the two pre-activations and insert $\widehat A^{(i)}h_{i-1}(x)$. This gives
\begin{align*}
 \widehat\eta_i(x)-\eta_i(x)
 ={}&\widehat A^{(i)}(\widehat h_{i-1}(x)-h_{i-1}(x))\\
 &+(\widehat A^{(i)}-A^{(i)})h_{i-1}(x)
 +(\widehat{\boldsymbol b}^{(i)}-\boldsymbol b^{(i)}).
\end{align*}
The induction hypotheses and the triangle inequality imply
\begin{align*}
 \|\widehat h_i(x)-h_i(x)\|_2
 &\le\|\widehat\eta_i(x)-\eta_i(x)\|_2\\
 &\le\|\widehat A^{(i)}\|_2U_{i-1}+e_iB_{i-1}+q_i=U_i.
\end{align*}
This proves both hidden-layer bounds by induction.

At the output layer, write
\begin{align*}
 \widehat f(x)-f_\theta(x)
 ={}&\widehat A^{(r+1)}(\widehat h_r(x)-h_r(x))\\
 &+(\widehat A^{(r+1)}-A^{(r+1)})h_r(x)
 +(\widehat{\boldsymbol b}^{(r+1)}-\boldsymbol b^{(r+1)}).
\end{align*}
For each term involving a matrix, use $\|Bz\|_\infty\le\|B\|_{2\to\infty}\|z\|_2$. The bounds just proved give $\|\widehat f(x)-f_\theta(x)\|_\infty\le V$.

Let $y$ be the victim's unique winning class. For every $c\ne y$, we have
$\widehat f_y(x)-\widehat f_c(x)\ge[f_\theta(x)]_y-[f_\theta(x)]_c-2V\ge\mu(x)-2V$.
If $\mu(x)>2V$, every such difference is positive. The reconstructed classifier therefore also predicts $y$.
\end{proof}

This result connects aligned parameter errors with sufficient conditions for label preservation. It does not supply small errors for missing rows. Direction estimates alone also do not bound $e_i$ until magnitudes, signs, and coordinate matching have been resolved. Products of weight norms can make the bound uninformative even when measured agreement is high. Those limitations must be checked before using the estimate to certify a reconstructed prefix or a classifier.

\subsection{Bias Estimates from Dual Points}
\label{sec:mlrecon:bias}
For an exact dual point $x^*$ of neuron $(i,j)$,
\[
 A_j^{(i)}h_{i-1}(x^*)+b_j^{(i)}=0.
\]
Consequently, if $\widehat A_j^{(i)}=cA_j^{(i)}$ and the true prefix features are known in matching coordinates, the compatible bias is $\widehat b_j^{(i)}=cb_j^{(i)}$. In the reconstructed network we instead define
\begin{equation}
\label{eq:bias-eq}
 \widehat b_j^{(i)}(x^*)=-\widehat A_j^{(i)}\widehat h_{i-1}(x^*).
\end{equation}
For each neuron with at least one assigned point, we take the median over at most 30 such estimates, proceeding from the first layer to the last. For a deeper layer, this is an estimate relative to the reconstructed prefix, not an unconditional recovery of the true bias.

To include approximate dual points, let $x_\ell$ be an assigned point and define its true critical residual by $r_\ell=A_j^{(i)}h_{i-1}(x_\ell)+b_j^{(i)}$. Write $\widehat A_j^{(i)}=cA_j^{(i)}+e^{\mathsf T}$ and $\widehat h_{i-1}(x_\ell)=h_{i-1}(x_\ell)+\delta_\ell$. Substituting into \textup{\eqref{eq:bias-eq}} gives
\[
 \widehat b_j^{(i)}(x_\ell)-cb_j^{(i)}
 =-cr_\ell-cA_j^{(i)}\delta_\ell
 -e^{\mathsf T}h_{i-1}(x_\ell)-e^{\mathsf T}\delta_\ell.
\]
Thus its absolute value is at most
\[
 |c|\bigl(|r_\ell|+\|A_j^{(i)}\|_2\|\delta_\ell\|_2\bigr)
 +\|e\|_2\bigl(\|h_{i-1}(x_\ell)\|_2+\|\delta_\ell\|_2\bigr).
\]
The critical residual vanishes at an exact point. If more than half the estimates have this bound at most $\epsilon$, their median lies in $[cb_j^{(i)}-\epsilon,cb_j^{(i)}+\epsilon]$. To prove this, order the estimates. Fewer than half lie below the interval and fewer than half lie above it. The middle estimate for an odd sample size therefore belongs to the interval. For an even sample size, both middle estimates belong to it, so their arithmetic mean also belongs to it. This is the median used by the implementation. The argument permits a minority of arbitrary outliers, but does not remove a common prefix error or a common critical-point residual. Later training may reduce prediction loss without correcting the true bias.

\subsection{Output-Layer Fitting}
\label{sec:mlrecon:output}
The output layer introduces no activation kink of its own, so hidden-neuron signature recovery does not directly estimate its rows. Its class differences do affect decision boundaries and can be studied by other methods~\cite{canalesmartinez2025somelayers}. Here we fit a multinomial logistic regression model using $z_\ell=\widehat h_r(x_\ell)$ and cached labels $y_\ell$. For weights $W$ and biases $b$, the data loss is
\[
 \mathcal L(W,b)=\frac1L\sum_{\ell=1}^L\left(
 \log\sum_{c=1}^K\exp((Wz_\ell+b)_c)
 -(Wz_\ell+b)_{y_\ell}\right).
\]
A finite fit requires an appropriate stopping rule or regularization, especially for separable data. We use scikit-learn's multinomial \texttt{LogisticRegression} with the \texttt{lbfgs} solver, an intercept, $C=10^6$, and a maximum of 2000 iterations. This is weakly regularized fitting, rather than an unregularized identification theorem. If some classes are absent, copying zeros into their output rows does not prevent those classes from winning the reconstructed argmax. A coverage check or explicit restriction to fitted classes is needed. The fitted coefficients define a decoder for the reconstructed features. They are not identified as the victim's unique coefficients from labels.

When sign--refinement cycles are enabled, we fit a provisional output layer before the first sign pass, alternate search with 20-epoch refinement bursts, and fit the output layer again afterward. With zero cycles, sign search instead precedes the output fit. Tables~\ref{tab:params-arch} and \ref{tab:params-shared} specify the reported cyclic runs. Both ordering choices use labels rather than victim logits, but repeated function calls can recache those labels and must be counted separately.

\subsection{Gradient Masking and Exact Row Preservation}
\label{sec:mlrecon:frozen}
Let $R_i$ be the indices of accepted rows in layer $i$, and let $\widehat A_j^{(i)}$ denote each selected row after sign assignment and scale selection. We write $W^{(i)}$ for the current trainable matrix of the reconstructed network, to distinguish it from the victim matrix $A^{(i)}$. The intended exact constraint for refinement is
$ W_j^{(i)}=\widehat A_j^{(i)}\ (j\in R_i)$.
The biases, output layer, and unaccepted rows remain trainable. An update satisfying this constraint retains the selected directions by construction. It does not prove that they were recovered without privileged information, that every retained sign is correct, or that the final hidden features equal the victim's features.

The implementation masks recovered-row gradients before optimization. We analyze AdamW with decoupled weight decay~\cite{loshchilov2019adamw}. Adam without weight decay is included by setting the decay coefficient to zero. Coupled weight decay inside the Adam gradient is a different update and is not covered by the following argument.

\begin{proposition}
\label{prop:row-preservation}
Suppose a nonzero recovered row $w_0$ is masked at every optimizer step, its Adam first and second moments initially vanish, and no operation adds an unmasked update to it. Let $\eta_t\ge0$ be the step size and $\lambda_t\ge0$ its decoupled weight-decay coefficient. If $0\le\eta_t\lambda_t<1$, then, without row restoration and in exact arithmetic, after $T$ steps
\[
 w_T=\left(\prod_{t=0}^{T-1}(1-\eta_t\lambda_t)\right)w_0.
\]
Its signed direction is preserved. Its exact value is preserved when $\lambda_t=0$ for all steps, or when the saved row is restored after each update.
\end{proposition}
\begin{proof}
Let $g_t$ be the masked gradient of the recovered row at step $t$. The masking assumption gives $g_t=0$. Write $m_t$ and $v_t$ for its first and second Adam moments at the start of that step, with $m_0=v_0=0$. For Adam coefficients $\beta_1,\beta_2\in[0,1)$, the moment updates are
\[
 m_{t+1}=\beta_1m_t+(1-\beta_1)g_t,\qquad
 v_{t+1}=\beta_2v_t+(1-\beta_2)(g_t\odot g_t),
\]
where $\odot$ denotes coordinatewise multiplication. If $m_t=v_t=0$, both updated moments are zero. Induction therefore gives $m_t=v_t=0$ at every step. Let $\epsilon_{\mathrm{Adam}}>0$ be the denominator offset. The bias-corrected moments are zero because their denominators $1-\beta_1^{t+1}$ and $1-\beta_2^{t+1}$ are positive. Each adaptive-update coordinate is therefore $0/(\sqrt{0}+\epsilon_{\mathrm{Adam}})=0$.

Only the decoupled weight-decay term remains. Thus
$w_{t+1}=w_t-\eta_t\lambda_t w_t=(1-\eta_t\lambda_t)w_t$.
Iterating from $t=0$ to $T-1$ gives the displayed product. Every factor is strictly positive by hypothesis, so the product is positive and $w_T/\|w_T\|_2=w_0/\|w_0\|_2$. This proves signed-direction preservation. If every $\lambda_t$ is zero, the product equals one and $w_T=w_0$. If the row is copied back from its saved value after each update, every completed step instead ends with $w_{t+1}=w_0$, independently of the tentative decayed value. This proves the restoration assertion.
\end{proof}

This proposition distinguishes fixed directions from fixed row values. Nonzero inherited moments, unmasked mini-refinement, or later sign changes require separate handling. Moreover, training a bias while rescaling its row generally changes its critical hyperplane. Direction preservation does not imply preservation of the original hidden feature map.

We can enforce the constraint without assumptions on the optimizer moments. Let $P_i$ be the diagonal matrix whose entry is one on $R_i$ and zero elsewhere, and let $W_{\mathrm{lock}}^{(i)}$ store the selected rows at entry to the refinement call. After any proposed update producing $\widetilde W^{(i)}$, define
\begin{equation}
\label{eq:row-projection}
 W^{(i)}=(I-P_i)\widetilde W^{(i)}+P_iW_{\mathrm{lock}}^{(i)}.
\end{equation}
Multiplying by $P_i$ and using $P_i^2=P_i$ gives $P_iW^{(i)}=P_iW_{\mathrm{lock}}^{(i)}$. Thus every projected step satisfies the constraint regardless of weight decay or inherited moments. In floating-point code, copying the saved coordinates back after each update preserves their bit patterns. A checkpoint restored by early stopping must satisfy the same constraint. Intentional sign changes begin a new refinement call with newly selected locked rows. This projection does not prove convergence to a global minimum of the label loss.

The complete ablation update masks gradients and applies AdamW without this restoration. In 906 nonempty layer comparisons immediately before and after refinement, every recovered-row group changed. The largest absolute coordinate change was approximately $1.23241\times10^{-4}$, while the smallest computed row cosine was $0.9999999999999996$. These measurements are consistent with direction retention under fresh moments, not exact row equality. The separate 288-configuration refinement study reports exact preservation in its 144 restoration configurations. Its numeric trajectories contain at least the reported number of invariant rows in each corresponding layer. A direct check of the intended row indices additionally requires exporting the recovered masks independently of the model-object checkpoints. The restoration study does not change the update used by the 48-run ablation sweep.

We also make the checkpoint-selection guarantee explicit. Let $G_0$ be the model entering a refinement call, let $G_1,\ldots,G_T$ be its checked checkpoints, and let $S$ be a fixed nonempty stopping sample with cached victim labels $y(x)$. Define $\widehat p_S(G)=|S|^{-1}\sum_{x\in S}\mathbf1[G(x)=y(x)]$, where $G(x)$ denotes the model's predicted class. If we retain a maximizer $G_*$ over $G_0,\ldots,G_T$, then $\widehat p_S(G_*)\ge\widehat p_S(G_0)$ because $G_0$ belongs to the candidate set. We implement this rule by evaluating and saving $G_0$ before the first optimizer step and restoring the best checked state. If the initial state is excluded, this inequality need not hold. The guarantee concerns the stopping sample only. The selected model may agree less often on an independent test sample. The fixed-family guarantee in Proposition~\ref{prop:sign-selection} can provide an additional population bound only when its independence and family assumptions hold. Reusing stopping labels to construct the checkpoint path does not meet those assumptions automatically.

\section{Experimental Results}
\label{sec:exp}
We compare assisted reconstruction with distillation and examine the individual shortcuts in Section~\ref{sec:exp:audited}. We call fitting on $X_{\mathrm{test}}\cup X_{\mathrm{test2}}$ the joint-fit protocol. The split-fit protocol fits on $X_{\mathrm{test}}$, scores sign candidates on $X_{\mathrm{test2}}$, and reserves $X_{\mathrm{test3}}$ for final reporting. Its refinement also uses the scoring slice for stopping unless a separate stopping pool is specified. Direction coverage uses the extraction procedures and privileged checks of Section~\ref{sec:cheats}. The comparison therefore does not measure the benefit of a strict hard-label cryptanalytic attack at equal information access.

Table~\ref{tab:datacontract} describes three data slices. The first two provide cached labels for fitting and search. The third is also used by the early-stopping watchdog in Table~\ref{tab:params-shared}. It is disjoint from the training inputs but participates in model selection. We consequently call it the evaluation and selection slice, rather than an untouched test set. An independent final test set would be needed for evaluation without this reuse.

\begin{table}[H]
\centering
\caption{Dataset contract for joint-fit experiments. The split-fit comparison uses the separate roles described in Section~\ref{sec:exp:audited}.}
\label{tab:datacontract}
\begin{tabular}{lll}
\toprule
Slice & Role & Source (\protect\path{make_blobs} / CIFAR-10) \\
\midrule
$X_{\mathrm{test}}$  & Phase-3 training pool                   & seed 42 / test batch\\
$X_{\mathrm{test2}}$ & Phase-3 training pool                   & seed 99 / train idx 40000--49999\\
$X_{\mathrm{test3}}$ & Evaluation and early stopping& seed 123 / train idx 10000--19999\\
\bottomrule
\end{tabular}
\end{table}

Tables~\ref{tab:params-arch} and \ref{tab:params-shared} give the settings recorded in the manuscript. The synthetic victims use data generated by \protect\path{make_blobs}. The CIFAR-10 runs use a separate driver and a V100 GPU with float64 arithmetic. The settings were chosen within the available computation budget. They do not establish optimal query or computation complexity.

The synthetic-data contract fixes the cluster centers at seed 42, uses the same scaling transform for all slices, and changes only the sampling seeds 42, 99, and 123. The inputs are clipped and scaled by $x\mapsto\operatorname{clip}(x,-3,3)/3$. Thus different slice seeds do not create different randomly chosen cluster centers. The six-victim batch log records 18:11:05--18:25:24 UTC on 21 June 2026, an elapsed time of 14 minutes 19 seconds including both search arms and reporting. This is a measured run duration, not a controlled speedup ratio. We describe the shortcut-comparison experiment separately in Section~\ref{sec:exp:audited}.

\begin{table}[H]
\centering
\caption{Per-architecture tuned parameters. The CIFAR runs use simulated annealing (SA) with the margin score. PT denotes parallel tempering.}
\label{tab:params-arch}
\setlength{\tabcolsep}{5pt} 
\footnotesize
\resizebox{0.7\linewidth}{!}{%
\begin{tabular}{lcccc}
\toprule
Parameter & \texttt{tiniest} & \texttt{tinier} & \texttt{tiny} & \texttt{full} (CIFAR) \\
\midrule
dual-search rounds        & 6       & 8       & 20      & 150 \\
dual workers / batch      & 7 / 256 & 7 / 256 & 7 / 256 & 1--3 / 256 \\
target triplets / round   & 3000    & 2000    & 10000   & 10000 \\
sign restarts             & 1       & 1       & 2       & 0--1 \\
sign pair-lookahead       & 8       & 8       & 8       & 4 \\
sign refine cycles        & 3       & 3       & 3       & 1 \\
refine epochs             & 300     & 500     & 500     & 500 \\
\bottomrule
\end{tabular}
}
\end{table}

\begin{table}[H]
\centering
\caption{Pipeline parameters held fixed across all victims.}
\label{tab:params-shared}
\setlength{\tabcolsep}{3pt}\footnotesize
\resizebox{0.8\textwidth}{!}{%
\begin{tabular}{lll}
\toprule
Phase & Parameter & Value \\
\midrule
Dual search  & implementation      & \texttt{torch}, CUDA float64 \\
Sign recovery & runner             & \protect\path{batched_sign_recovery.py}, float64 \\
Reconstruction & sign-search method & SA+margin (default), PT+margin (comparison) \\
             & mini-refine          & 20 epochs, lr $5{\times}10^{-3}$, wd $10^{-4}$ \\
             & refine optimiser     & AdamW, cosine lr, wd $10^{-4}$ \\
             & early stopping        & patience 5, evaluation every 10 epochs \\
             & data use             & select on X\_test3, train on X\_test $\cup$ X\_test2 \\
Activation   & \protect\path{LEAKY_ALPHA} & 0.0 (ReLU) / 0.01 (Leaky ReLU) \\
Label access & reconstruction batches & 15 measured calls in the seeded synthetic sweep \\
\bottomrule
\end{tabular}%
}
\end{table}

We count one hard-label request for each input whose victim label is requested. If stage $j$ calls the victim on a batch $B_j$, the request count is $\sum_j|B_j|$. Repeated requests count again. We distinguish this count from batched calls and from the number of distinct inputs labelled. The seeded synthetic sweep records 15 batched calls per phase-3 execution, including the stopping-set labels. Under the joint-fit protocol, each fitting or scoring input is labelled 14 times. Across the six victims, phase 3 requests 57,000--211,024 labels under that protocol and 29,024--96,024 under the split-fit protocols. These are phase-3 costs only. Section~\ref{sec:exp:seeded} reports the additional phase-1 requests and the privileged accesses separately. Caching could reduce repeated label requests, but the measured implementation does not use one cache across all these calls.

\subsection{Evaluation Metric}
\label{sec:exp:eqs-intro}
We write $\widehat{\mathcal O}(x)=\arg\max_c[\widehat f(x)]_c$ for a reconstructed classifier. For a finite nonempty set $X$, define
\[
 C(X)=\frac1{|X|}\sum_{x\in X}\mathbf1[\widehat{\mathcal O}(x)=\mathcal O(x)].
\]
We report three functional metrics. $C_1=C(X_{\mathrm{test3}})$ measures agreement on the evaluation and selection slice in the joint-fit protocol. In the split-fit comparison, this slice is reserved for final reporting. For $C_2$, we generate two pools of 5000 probes. The uniform pool uses the coordinatewise bounding box of the reference slice $X_{\mathrm{test}}$. The Gaussian pool has the reference mean and coordinatewise standard deviations multiplied by 3, with a $10^{-12}$ guard on the standard deviations. We define
\[
 C_2=\tfrac12\bigl(C(X_{\mathrm{uniform}})+C(X_{\mathrm{Gaussian}})\bigr).
\]
Interpolation-path agreement is reported separately and is not included in this average.

For $C_3$, we select at most 1500 evaluation inputs and sample eight unit directions at each point. For each direction, we query the label at radius 2. If that label differs from the starting label, we bisect the segment 12 times while maintaining endpoints with different labels and record the upper endpoint radius. Otherwise the direction is assigned infinite radius. The minimum directional result is clipped at 2, giving a proxy $\widehat d(x)$. This is a label-based search radius, not the exact distance to the nearest decision boundary. Without monotonicity of labels along a ray, bisection need not locate its first crossing, and matching endpoint labels can conceal intermediate crossings.

We bin the proxy by its empirical tertiles, deduplicating repeated quantile edges, and define $C_3$ as agreement in the resulting far bin. If that bin is absent, the implementation uses $C_1$. Ties and clipping mean that the far bin need not contain one third of the subsample. For instance, the CIFAR ReLU and Leaky ReLU far bins contain 1349 and 1327 of 1500 inputs. This input-space proxy differs from the reconstructed logit margin used by sign search.

\begin{table}[H]
\centering
\caption{Functional agreement metrics. None identifies the victim's parameters.}
\label{tab:eqs-metrics}
\setlength{\tabcolsep}{4pt} 
\footnotesize
\resizebox{0.8\linewidth}{!}{%
\begin{tabular}{p{4cm}p{7cm}}
\toprule
Metric & Interpretation \\
\midrule
$C_1$ -- Evaluation-slice agreement & Prediction agreement on the slice also used for early stopping. \\
$C_2$ -- Off-distribution agreement & Prediction agreement under the specified alternative input distributions. \\
$C_3$ -- Margin-conditioned agreement & Prediction agreement in the far bin of the directional search-radius proxy. \\
\bottomrule
\end{tabular}
}
\end{table}

Off-distribution agreement can reveal differences missed by a single evaluation distribution, but high $C_2$ does not prove structural recovery. A distilled model can also agree away from its training samples. Likewise, conditioning on margin does not remove all boundary-estimation or sampling effects. Structural evaluation instead requires direct parameter comparisons after accounting for neuron permutations and scaling ambiguities.

\subsection{Per-Victim Extraction Results}
\label{sec:exp:permodel}
Table~\ref{tab:permodel} gives accepted directions per layer, the reported sign accuracy, task accuracy, and agreement with the victim. The four hidden layers are denoted fc1--fc4. Direction counts incorporate D5 candidate checking. We omit a direction-cosine column because ground-truth acceptance would make that statistic conditional on the check itself.

For a layer with $r_i>0$ accepted rows, define $s_i$ as the fraction whose dot product with the matched true row is positive. The sign column reports the unweighted mean over nonempty layers,
\[
 \overline s=\frac1{|I|}\sum_{i\in I}s_i,\qquad I=\{i:r_i>0\}.
\]
It is not the pooled neuron fraction $\sum_i r_i s_i/\sum_i r_i$. For example, the first row averages $3/8$, $3/7$, and $5/5$, giving $0.601190\ldots$, whereas the pooled fraction is $11/20=0.55$. This resolves the nonintegral-looking sign entry. Neuron matching follows the assisted extraction metadata.

\begin{table}[H]
\centering
\caption{Assisted reconstruction across the victim suite (SA with a margin score). Sign accuracy averages over nonempty hidden layers.}
\label{tab:permodel}
\setlength{\tabcolsep}{4pt}\scriptsize
\resizebox{\textwidth}{!}{%
\begin{tabular}{lllcccc}
\toprule
Model & Recovered & fc1/fc2/fc3/fc4 & \makecell{Sign\\acc} & \makecell{Victim\\acc} & \makecell{Ext\\acc} & Agreement \\
\midrule
tiniest\_relu      & 20/32 (62.5\%)   & 8/8 $\cdot$ 7/8 $\cdot$ 0/8 $\cdot$ 5/8        & 0.601 & 99.90\%  & 99.40\%  & 99.30\%  \\
tiniest\_leakyrelu & 25/32 (78.1\%)   & 8/8 $\cdot$ 8/8 $\cdot$ 3/8 $\cdot$ 6/8        & 0.771 & 100.00\% & 100.00\% & 100.00\% \\
tinier\_relu       & 29/56 (51.8\%)   & 16/16 $\cdot$ 12/16 $\cdot$ 1/16 $\cdot$ 0/8  & 0.347 & 100.00\% & 100.00\% & 100.00\% \\
tinier\_leakyrelu  & 35/56 (62.5\%)   & 16/16 $\cdot$ 13/16 $\cdot$ 4/16 $\cdot$ 2/8  & 0.463 & 100.00\% & 100.00\% & 100.00\% \\
tiny\_relu         & 139/256 (54.3\%) & 64/64 $\cdot$ 61/64 $\cdot$ 14/64 $\cdot$ 0/64  & 0.503 & 100.00\% & 100.00\% & 100.00\% \\
tiny\_leakyrelu    & 226/256 (88.3\%) & 63/64 $\cdot$ 60/64 $\cdot$ 49/64 $\cdot$ 54/64 & 0.500 & 100.00\% & 100.00\% & 100.00\% \\
cifar\_relu        & 499/832 (60.0\%) & 252/256 $\cdot$ 247/256 $\cdot$ 0/256 $\cdot$ 0/64    & 0.531 & 51.79\% & 46.60\% & 56.40\% \\
cifar\_leakyrelu   & 807/832 (97.0\%) & 254/256 $\cdot$ 252/256 $\cdot$ 243/256 $\cdot$ 58/64 & 0.537 & 51.18\% & 46.60\% & 58.29\% \\
\bottomrule
\end{tabular}%
}
\end{table}
The six synthetic-data victims have agreement of at least $99.30\%$ in Table~\ref{tab:permodel}. In particular, \protect\path{tiny_relu} reaches $100\%$ agreement with 139 of 256 directions accepted. Thus complete direction coverage is not necessary for agreement on these inputs. Conversely, \protect\path{cifar_leakyrelu} has 807 of 832 directions accepted but only $58.29\%$ agreement. High direction coverage is not sufficient for high functional agreement.

Task accuracy compares predictions with dataset labels, whereas agreement compares them with victim predictions. They need not coincide because the victim itself makes errors. The distillation comparison below evaluates prediction agreement and does not identify which retained parameters cause a difference.

\subsection{ReLU versus Leaky ReLU Recovery}
\label{sec:exp:scaling}
Table~\ref{tab:scaling} compares accepted directions. Leaky ReLU has 5, 6, 87, and 308 more accepted directions on the four architectures, respectively. The activation, trained parameters, sampled geometry, and acceptance rule all affect this comparison. The counts therefore do not isolate the effect of the negative-side slope alone. The differences also occur across several layers, rather than only at fc4.

\begin{table}[H]
\centering
\caption{Phase 1--2 recovery, ReLU vs Leaky ReLU.}
\label{tab:scaling}
\footnotesize
\resizebox{0.6\linewidth}{!}{%
\begin{tabular}{lccccc}
\toprule
Architecture & \makecell{ReLU\\recovered} & \makecell{Leaky\\recovered} & Delta & fc4 ReLU & fc4 Leaky \\
\midrule
tiniest & 20/32  & 25/32  & $+5$  & 5/8  & 6/8  \\
tinier  & 29/56  & 35/56  & $+6$  & 0/8  & 2/8  \\
tiny    & 139/256 & 226/256 & $+87$ & 0/64 & 54/64 \\
cifar   & 499/832 & 807/832 & $+308$ & 0/64 & 58/64 \\
\bottomrule
\end{tabular}
}
\end{table}

\subsection{Depth and the Deepest Hidden Layer} 
\label{sec:exp:depth} 

In the \protect\path{tiniest_relu} run in Table~\ref{tab:permodel}, fc3 has no accepted directions and fc4 has five. Table~\ref{tab:depth-runs} gives an aggregate summary of 40 repeated extractions. In that summary, fc4 has at least as many accepted directions as fc3 in all 40 runs and strictly more in 35. These counts describe the sampled runs and do not imply monotonicity of coverage with depth. The table gives aggregate counts rather than run-level variation. We examine variation across recorded seeds separately in Section~\ref{sec:exp:seeded}. An unseeded search does not by itself establish independence between runs.

\begin{table}[H] 
\centering 
\caption{Per-layer recovered neurons over $40$ reported repeated extractions of \protect\path{tiniest_relu}.} 
\label{tab:depth-runs} 
\footnotesize
\resizebox{0.4\linewidth}{!}{%
\begin{tabular}{lccc} \toprule Layer & Mean & SD & Range \\ \midrule fc1 & 8.00 & 0.00 & 8 \\ fc2 & 6.97 & 0.16 & 6--7 \\ fc3 & 2.25 & 1.37 & 0--5 \\ fc4 & 4.60 & 0.54 & 3--5 \\ 
\bottomrule
 \end{tabular}
 }
 \end{table}

To examine inactive coordinates, the reported probe set $\mathcal P$ contains $X_{\mathrm{test3}}$ and $200{,}000$ uniform samples from a threefold-inflated data bounding box. We define
\[
 Z_i(\mathcal P)=\{j:\eta_{i,j}(x)\le0\text{ for every }x\in\mathcal P\}.
\]
This records inactivity on a finite probe set, not permanent inactivity over the input space or over all extraction points.

For a particular cluster at layer $i$, let $M_i$ have rows equal to its transported tangent directions in $h_{i-1}$ coordinates. Let $Z$ contain coordinates of $h_{i-1}$ that are identically zero in the neighborhoods used by all those tangents. Under ReLU, column $j$ of $M_i$ is then zero for every $j\in Z$. Writing $e_j$ for the standard coordinate vector gives
\[
 \operatorname{span}\{e_j:j\in Z\}\subseteq\ker M_i.
\]
If the cluster is exact, its true weight row $a$ also satisfies $M_i a^{\mathsf T}=0$. Therefore
\begin{equation}
\label{eq:kernel-bound}
 \dim\ker M_i\ge |Z|+
 \mathbf1\left[a^{\mathsf T}\notin\operatorname{span}\{e_j:j\in Z\}\right].
\end{equation}
Indeed, the coordinate vectors are independent, and $a^{\mathsf T}$ contributes one further independent kernel vector precisely when it is outside their span. The stronger bound $1+|Z|$ needs this additional condition. If $|Z|\ge1$ and that condition holds, a test requiring a one-dimensional kernel rejects the cluster. Probe-set inactivity alone does not establish the premises of this argument.

The reported probe counts $|Z_3(\mathcal P)|$ are 0, 1, and 2 for the three ReLU synthetic architectures. Their fc4 coverage in Table~\ref{tab:scaling} is 5/8, 0/8, and 0/64. This is consistent with inactive coordinates reducing available constraints, but does not prove that the probe counts cause the observed rejection. Under Leaky ReLU, a strictly negative coordinate has nonzero slope $\alpha$, so this particular zero-column argument no longer applies. Other rank deficiencies remain possible.

The three reported \protect\path{tiny_leakyrelu} runs have fc4 counts 54, 56, and 57 and corresponding fc3 counts 49, 50, and 47. These observations show that accepted direction counts need not decrease with depth.

\subsection{Two-Arm Metric Comparison}\label{sec:exp:eqs}
Table~\ref{tab:eqs} compares assisted reconstruction (ext) with a learning-only distillation baseline (dis). The baseline is described as using the same label-query budget. A complete budget comparison must also count extraction and evaluation queries per input and disclose the assisted arm's privileged accesses. Matching cached labels alone does not match the information used by the two procedures.
\begin{table}[H]
\centering
\caption{Per-metric agreement for the extraction arm and a matched distillation baseline.}
\label{tab:eqs}
\footnotesize
\resizebox{0.7\linewidth}{!}{%
\begin{tabular}{lcccccc}
\toprule
Model & \multicolumn{2}{c}{$C_1$} & \multicolumn{2}{c}{$C_2$} & \multicolumn{2}{c}{$C_3$} \\
 & ext & dis & ext & dis & ext & dis \\
\midrule
tiniest\_relu      & 0.9930 & 0.9610 & 0.2166 & 0.3304 & 0.9910 & 0.9731 \\
tiniest\_leakyrelu & 1.0000 & 0.9860 & 0.5880 & 0.3756 & 1.0000 & 1.0000 \\
tinier\_relu       & 1.0000 & 1.0000 & 0.5297 & 0.5295 & 1.0000 & 1.0000 \\
tinier\_leakyrelu  & 1.0000 & 1.0000 & 0.6123 & 0.5132 & 1.0000 & 1.0000 \\
tiny\_relu (SA)    & 1.0000 & 1.0000 & 0.4335 & 0.3338 & 1.0000 & 1.0000 \\
tiny\_relu (PT)    & 1.0000 & 1.0000 & 0.4900 & 0.3338 & 1.0000 & 1.0000 \\
tiny\_leakyrelu    & 1.0000 & 1.0000 & 0.2835 & 0.3741 & 1.0000 & 1.0000 \\
cifar\_relu        & 0.5640 & 0.5033 & 0.3365 & 0.2173 & 0.6012 & 0.5271 \\
cifar\_leakyrelu   & 0.5829 & 0.5286 & 0.3383 & 0.2246 & 0.6292 & 0.5486 \\
\bottomrule
\end{tabular}
}
\end{table}
On five of the nine rows, both $C_1$ and $C_3$ equal $1.0000$ in both arms, so only $C_2$ separates their reported scores. Distillation exceeds assisted reconstruction on $C_2$ by $0.1138$ for \protect\path{tiniest_relu} and $0.0906$ for \protect\path{tiny_leakyrelu}. The synthetic experiments therefore do not show a uniform reconstruction advantage.

All three metrics favor assisted reconstruction on both CIFAR victims. The differences are $0.0607$, $0.1192$, and $0.0741$ for ReLU and $0.0543$, $0.1137$, and $0.0806$ for Leaky ReLU. These are differences in the reported scores. The CIFAR scorecards also record paired discordant counts against the victim reference. On ReLU, the counts of inputs matched only by reconstruction and only by distillation are $b=1851$ and $c=1244$. For Leaky ReLU they are $b=1852$ and $c=1309$. The continuity-corrected McNemar statistic is $(|b-c|-1)^2/(b+c)$, giving $118.655$ and $92.934$. The corresponding one-degree-of-freedom chi-square tail probabilities are approximately $1.25\times10^{-27}$ and $5.41\times10^{-22}$. These summarize paired disagreement in the reported single runs. They do not measure variation across training seeds, and selection using the evaluation slice limits their interpretation as independent-test significance. The reported bootstrap gap intervals are $[4.99,7.18]$ and $[4.36,6.55]$ percentage points. We interpret these intervals as within-run summaries, rather than uncertainty estimates across independently trained victims.

The simulated-annealing and parallel-tempering runs reuse the same extracted artifacts on each synthetic victim. Five victims have the same reported final sign mean in both arms. On \protect\path{tiny_relu}, the mean changes from $0.502756\ldots$ to $0.518381\ldots$, while $C_1$ and $C_3$ remain 1. The off-distribution score changes from $0.4335$ to $0.4900$. These are single-run comparisons. The search can terminate immediately when its scoring agreement is saturated, which limits how much such a comparison tests the two optimizers.

The evaluation scripts also form a composite score from functional agreement and an extraction-specific structural term. We retain the component comparisons in Table~\ref{tab:eqs} rather than use that composite as evidence that distillation cannot recover structure. Its structural term averages layerwise absolute cosine, the sign mean, and accepted-row coverage, and assigns the distillation arm zero because it has no extraction records. That zero is an accounting convention, not a measured absence of aligned parameters. Assigning it creates a built-in advantage for the extraction arm. A fair structural comparison would independently match and assess the distillation weights, allowing the same permutation and scale ambiguities.

\subsection{Additive Reconstruction Ablation}
\label{sec:exp:ablation}
Table~\ref{tab:ablation} builds the Phase-3 reconstruction one component at a time on each \protect\path{make_blobs} victim, recording agreement on the evaluation and selection slice at every stage. This ablation study is a separate execution of the pipeline, thus its final row need not coincide exactly with Table~\ref{tab:permodel}.

\begin{table}[H]
\centering
\caption{Additive Phase-3 ablation on the six \protect\path{make_blobs} victims, evaluated on the selection slice X\_test3.}
\label{tab:ablation}
\setlength{\tabcolsep}{4pt}\footnotesize
\resizebox{0.7\linewidth}{!}{%
\begin{tabular}{llcc}
\toprule
Arch & Stage & ReLU Agr & Leaky ReLU Agr \\
\midrule
\multirow{6}{*}{tiniest}
 & RAW (Phase 1+2)        & 12.40\% & 12.50\% \\
 & + bias recovery        & 22.80\% & 2.80\% \\
 & + output-layer fit     & 92.40\% & 98.00\% \\
 & + sign search          & 99.40\% & 97.90\% \\
 & + masked refinement & 99.70\% & 98.90\% \\
 & distillation baseline  & 96.00\% & 98.60\% \\
\midrule
\multirow{6}{*}{tinier}
 & RAW (Phase 1+2)        & 23.30\% & 25.00\% \\
 & + bias recovery        & 0.00\% & 6.58\% \\
 & + output-layer fit     & 99.94\% & 99.92\% \\
 & + sign search          & 100.00\% & 100.00\% \\
 & + masked refinement & 100.00\% & 100.00\% \\
 & distillation baseline  & 100.00\% & 100.00\% \\
\midrule
\multirow{6}{*}{tiny}
 & RAW (Phase 1+2)        & 3.00\% & 10.68\% \\
 & + bias recovery        & 14.34\% & 12.60\% \\
 & + output-layer fit     & 100.00\% & 100.00\% \\
 & + sign search          & 100.00\% & 100.00\% \\
 & + masked refinement & 100.00\% & 100.00\% \\
 & distillation baseline  & 100.00\% & 100.00\% \\
\bottomrule
\end{tabular}
}
\end{table}

Output-layer fitting produces the largest agreement increase in this ablation. For \protect\path{tiny_relu}, the score rises from $14.34\%$ after bias estimation to $100.00\%$ after the fit. The $3.00\%$ raw score precedes both steps. A successful decoder fit shows that the reconstructed features separate these labeled samples, not that they equal the victim's features.

For \texttt{tinier} and \texttt{tiny}, agreement after fitting is at least $99.92\%$, leaving at most $0.08$ percentage points before saturation. This score cannot usefully distinguish later improvements on those samples. On \texttt{tiniest}, sign search changes agreement by $+7.00$ percentage points under ReLU and $-0.10$ under Leaky ReLU. Refinement adds $+0.30$ and $+1.00$ points, respectively. The decrease on the Leaky ReLU row is possible because search selects by a separate cached scoring set, rather than guaranteeing improvement on the evaluation slice.

Bias estimation sometimes lowers agreement before the output layer is fitted. At that stage the decoder remains randomly initialized, so the score alone is not a reliable measure of the bias estimates. Finally, reconstruction and distillation tie on \texttt{tinier} and \texttt{tiny}. Reconstruction exceeds distillation by $3.70$ and $0.30$ percentage points on the two \texttt{tiniest} rows.

\subsection{Shortcut and Reconstruction Comparisons}
\label{sec:exp:audited}
We examine a separate batch of 48 synthetic-victim runs, comprising four shortcut comparisons, six victims, and two arms. We vary D1, D2, D3, and D5 individually, retaining the true-prefix and true-walk assistance. The ON configurations retain the corresponding assistance, and the OFF configurations disable the named shortcut. The batch evaluates two phase-3 protocols on each shared phase-1 and phase-2 output. The joint-fit protocol fits on the union of fitting and scoring inputs and uses the reporting slice for early stopping. The split-fit protocol fits on $X_{\mathrm{test}}$, scores on $X_{\mathrm{test2}}$, uses those two slices for refinement, and reserves $X_{\mathrm{test3}}$ for final reporting. This separation concerns data use in phase 3, not the privileged accesses in extraction.

The tiniest fitting, scoring, and reporting sets have 2,000, 2,000, and 1,000 inputs. The corresponding sizes for tinier and tiny are 10,000, 5,000, and 5,000. The recorded split-fit phase-3 access logs contain no requests outside their phase-specific input contract. No identical input rows occur across these slices. This overlap check does not by itself establish statistical independence. The input guard also returns logits even though the inspected consumers reduce them to predicted labels. A strict label-only interface must restrict the return value itself.

We check the predictions directly from the numeric checkpoint parameters. Across 1,344 models, all 4,032 prediction vectors on the three input slices agree with the checkpoint evaluations, and all 576 stage-agreement values agree with the recomputed values. Table~\ref{tab:audited-clean} gives one D1 ON trace per victim under the split-fit protocol. This comparison uses gradient masking without row restoration.

\begin{table}[H]
\centering
\caption{D1 ON split-fit reconstruction. Accepted rows are counted
across all hidden layers. Agreement is measured on the reserved
reporting slice.}
\label{tab:audited-clean}

\setlength{\tabcolsep}{4pt}
\footnotesize
\resizebox{0.7\linewidth}{!}{%
\begin{tabular}{@{}lcccc@{}}
\toprule
Victim & Accepted rows & Output fit & Sign search & Refinement \\
\midrule
tiniest\_relu      & 20/32   & 99.60\%  & 99.90\%  & 99.90\%  \\
tiniest\_leakyrelu & 20/32   & 97.90\%  & 99.50\%  & 99.20\%  \\
tinier\_relu       & 28/56   & 100.00\% & 100.00\% & 100.00\% \\
tinier\_leakyrelu  & 31/56   & 100.00\% & 100.00\% & 100.00\% \\
tiny\_relu         & 142/256 & 100.00\% & 100.00\% & 100.00\% \\
tiny\_leakyrelu    & 231/256 & 100.00\% & 100.00\% & 100.00\% \\
\bottomrule
\end{tabular}%
}
\end{table}

The D1 comparison on tiny ReLU accepts 142/256 rows with assistance and 24/256 without it. For tiny Leaky ReLU, the corresponding counts are 231/256 and 88/256. These differences describe the assisted implementation under distinct search realizations. The split-fit D1 tiniest Leaky ReLU ON run has final agreement $0.992$, while its joint-fit counterpart has agreement $0.995$. Thus a claim that every D1 run reaches $0.995$ is false across the two protocols.

The four ON configurations yield four canonical repetitions per victim. For tiniest ReLU, their fc1--fc4 means are $8$, $7$, $1$, and $4.75$, and their sample standard deviations are $0$, $0$, $1.1547$, and $0.5$. These observations show that accepted counts need not decrease at every layer. These four repetitions are distinct from the 40-run experiment in Table~\ref{tab:depth-runs}. The unseeded dual-point searches also prevent a paired interpretation of ON--OFF differences based on a common recorded random stream. Across the batch, 2,063 of 3,002 decided rows have correct recovered-frame votes. The operational orientation conversion uses victim rows as described in Section~\ref{sec:sign-recovery}.

The exhaustive sign-block experiment enumerates 256 assignments in each of 27 blocks. It considers a fixed decoder fitted on the union pool, a fixed decoder fitted on the fitting slice alone, and a decoder refitted on that slice for every assignment. The initial assignment maximizes scoring agreement, allowing ties, in 24/27 blocks with either fixed decoder and in 3/27 with refitting. All 81 block/mode agreement records recompute exactly. In six selected blocks, replacing scoring and reporting labels with random labels leaves the stored decoder coefficients and candidate predictions unchanged in all 18 checked cases. This supports separation of candidate construction from scoring for those block procedures. Proposition~\ref{prop:sign-selection} additionally requires the stated sampling assumptions. The union-fitted decoder depends on scoring labels and does not meet that proposition's fixed-family condition. These block tests also do not establish global parameter identification or the optimization accuracy of the full adaptive search.

The split-fit phase-3 logs record 720 calls and 3,537,152 input evaluations across the 48 runs, ranging from 29,024 to 96,024 evaluations per run. Each repeated input is counted again. The harness passes the constant 20,000 as both its budget and its used-query count. That constant does not equal these recorded totals. We must distinguish pool size, distinct queried inputs, repeated evaluations, and privileged reads before making a comparison at equal attack cost. A phase-3 ledger also does not supply the entire cost of phases 1 and 2.

\subsection{Row restoration, victim-free walks, and seeded repetitions}
\label{sec:exp:seeded}
We evaluate row restoration, replacement sign walks, refinement protocols, and recovery depth. The restoration and walk comparisons use shared extraction outputs to compare downstream procedures. The refinement and depth experiments vary recorded dual-search seeds. We specify the protocol, sample size, and information access for each comparison.

\paragraph{Exact row restoration.}
We compare gradient masking with the restoration update in Eq.~\eqref{eq:row-projection}, under the joint-fit and split-fit protocols on 48 extraction runs. All 48 restoration runs per protocol have zero recorded change in the accepted row coordinates after every recorded optimizer step, with equality at all 672 checked checkpoints. The restoration reference is reset at each refinement call, so intentional sign changes between calls remain possible. For masking alone, the largest recorded coordinate change is approximately $1.54\times10^{-4}$ under the joint-fit protocol and $1.23\times10^{-4}$ under the split-fit protocol. Restoration changes final agreement by at most $0.004$ and $0.003$, respectively, in these paired comparisons. These small observed differences do not establish statistical equivalence. Proposition~\ref{prop:row-preservation} proves the preservation property for the specified update independently of its effect on agreement.

\paragraph{Replacing victim-weight access in the walk.}
We compare the canonical walk with two replacements on the extraction outputs from 24 runs with the corresponding shortcuts enabled. One replacement uses a reconstructed hidden stack with the victim output layer. The other uses reconstructed hidden layers and a label-fitted output layer. The latter is the victim-free walk variant. The replacement uses a guard that blocks victim-model loading and parameter access. In the guard tests, replacing victim-model files leaves the walk outputs unchanged. Its input artifacts still originate from assisted extraction, so it is not an end-to-end hard-label experiment. Fitting its provisional heads reuses a label pool of 2,000 or 10,000 inputs. One batch call to obtain that pool counts as 2,000 or 10,000 label requests.

Across decided rows, the canonical walk has 1,255 correct votes out of 1,819, or $68.99\%$. The victim-free replacement has 836 out of 1,810, or $46.19\%$. The replacement retaining the victim output layer has 876 out of 1,820, or $48.13\%$. These denominators differ, so the pooled percentages are not a paired comparison on identical decided rows. Every replacement walk terminates by exhausting its available dual points. The resulting comparisons include incomplete walks. Exhausting a point pool is not a certified hitting-distance comparison. After split-fit reconstruction with row restoration, the victim-free variant nevertheless has mean final agreement $0.99658$ across the 24 runs. Thus high functional agreement coexists with poor sign accuracy in this experiment. Proposition~\ref{prop:local-vote-stability} specifies a sufficient local stability condition, but this experiment does not establish the required slope, perturbation, or gap bounds.

\paragraph{Refinement and independent reporting.}
We repeat the two tiniest victims with D1 enabled or disabled and dual-search seeds 1--10, giving 40 extraction runs. Each is evaluated under four phase-3 protocols. The joint-fit protocol fits on $X_{\mathrm{test}}\cup X_{\mathrm{test2}}$ and stops on the 1,000-input stopping set. The split-fit protocol fits on $X_{\mathrm{test}}$, selects signs on $X_{\mathrm{test2}}$, and also stops on that scoring set. Thus it separates reporting from selection, but does not give refinement a separate stopping set. A third variant adds exact row restoration and includes the entering model among the checked checkpoints. A fourth also uses a separate stopping pool $X_{\mathrm{val}}$ of 2,000 inputs, of which 1,024 are labelled for the stopping rule. All four variants are scored on the same held-out pool of 5,000 inputs, disjoint from the fitting, scoring, and stopping pools and unused during fitting or selection.

\begin{table}[H]
\centering
\caption{Final agreement on 5,000 held-out inputs across 40 seeded tiniest runs. The extraction and canonical sign walk retain privileged access.}
\label{tab:seeded-agreement}
\small
\begin{tabular}{lccc}
\toprule
Phase-3 protocol & Minimum & Maximum & Runs below $0.995$ \\
\midrule
Joint-fit & 0.9918 & 0.9998 & 4/40 \\
Split-fit & 0.9890 & 0.9992 & 15/40 \\
Split-fit, restoration and initial checkpoint & 0.9910 & 0.9996 & 16/40 \\
Also separate stopping pool & 0.9910 & 0.9998 & 13/40 \\
\bottomrule
\end{tabular}
\end{table}

In the D1 ON tiniest Leaky ReLU split-fit run ending at $0.992$ on 1,000 reporting inputs, refinement increases the errors from five to eight. On its stopping set, errors rise from ten to eighteen because the entering model was not eligible for selection. Including that model prevents this deterioration on the stopping sample. Across the 80 refinement executions with initial-checkpoint selection, the returned stopping agreement is never below its entering value. It is higher in 62 executions and equal in 18. The held-out benefit is smaller and not uniform. The variant with restoration and a separate stopping pool improves agreement over split-fit by $0.000425$ on average, with 22 runs improving, 14 worsening, and four unchanged. This comparison changes restoration and checkpoint selection together, so it does not isolate their separate effects.

The $0.9890$ minimum in Table~\ref{tab:seeded-agreement} applies to the 40 seeded runs in that experiment. The D1 ON tiniest Leaky ReLU model with agreement $0.992$ on 1,000 inputs has agreement $0.9882$ on the 5,000-input pool. Thus the minimum in the table is specific to its experimental configuration and seed set. Likewise, the joint-fit protocol's higher average score on its own stopping sample does not by itself quantify the causal effect of selection. Its score on that sample is not an independent test estimate. Both protocols' scores on the held-out pool avoid this direct reuse.

\paragraph{Recovery depth across recorded seeds.}
We run the canonical configuration on six victims with seeds 1--8, giving 48 extraction runs and three phase-3 protocols per run. The 16 tiniest extractions repeat point-for-point the corresponding seed cases in the 40-run study and are not additional independent repetitions. For these repeated seed cases, the dual points, recovered rows, and signs have matching hashes, and the phase-3 results agree. Search rounds use recorded seeds derived from the experiment seed and round index. These repetitions use the same machine, library versions, and thread setting. Reproducibility across other configurations is not established. The runs use a CPU, float64 arithmetic, three dual workers, and two threads per process. All shortcuts remain enabled. This seeded experiment is restricted to the six synthetic victims.

\begin{table}[H]
\centering
\caption{Accepted rows across eight recorded dual-search seeds per synthetic victim. Parentheses give means.}
\label{tab:seeded-depth}
\small
\begin{tabular}{lcc}
\toprule
Victim & Total accepted rows & Deepest hidden layer \\
\midrule
Tiniest ReLU & 20--23 (22.125) & 4--5 (4.875) \\
Tiniest Leaky ReLU & 21--25 (22.875) & 2--5 (3.250) \\
Tinier ReLU & 29--31 (30.000) & 0 (0.000) \\
Tinier Leaky ReLU & 31--34 (32.875) & 0--1 (0.250) \\
Tiny ReLU & 131--141 (136.000) & 0 (0.000) \\
Tiny Leaky ReLU & 221--229 (225.875) & 51--57 (54.375) \\
\bottomrule
\end{tabular}
\end{table}

The largest architecture consistently shows more accepted directions under Leaky ReLU, including substantial coverage in its deepest layer. On tiniest, the deepest-layer ordering is reversed on average. The single-run Leaky ReLU advantage in Table~\ref{tab:permodel} therefore does not imply a uniform advantage across layers or seeds. These are comparisons of fixed trained victims. They do not isolate activation choice from differences in training or learned parameters. The eight seeded runs per victim are distinct from the 40-run experiment in Table~\ref{tab:depth-runs}. All 32 tinier and tiny runs have final agreement 1 on their 5,000 reporting inputs under each protocol. This is agreement on those finite samples, not equality of the classifiers on every input.

\paragraph{Measured access to the victim.}
We classify each victim access by its returned information. Label-only responses count as hard-label requests. Margins, gradients, hidden activations, and parameters count as privileged access. Unclassified accesses are assigned to the privileged category. Across the 48 runs, the phase-1 counter matches the separate dual-search ledger for labels, detection margins, Newton margins, and gradients. For these seeded runs, the phase-3 guard encodes only the predicted label as a one-hot vector, with a single entry equal to one and all other entries zero. The counter checks find no untyped accesses or unassigned search rows. These consistency checks concern the implemented access paths. They do not establish that an arbitrary implementation has no other access path.

Phase 1 requests between 704,964 and 12,027,750 hard labels per run. Phase 3 makes 15 batched label calls per protocol execution. Nine calls score sign-search stages, two fit output layers, two support short refinements, one supports final refinement, and one obtains stopping labels. Under the joint-fit protocol, this totals 57,000 requests on tiniest and 211,024 on tinier or tiny. Its distinct phase-3 labelled-input counts are 5,000 and 16,024. Under the split-fit protocols, the corresponding repeated-request totals are 29,024 and 96,024. The canonical sign walk makes no label queries, but reads victim weights and uses them in its simulated evaluations. Zero label queries at that stage is therefore not evidence of zero privileged access.

We do not combine heterogeneous accesses into a hard-label query budget. Detection-margin and gradient channels record input rows, while parameter channels count scalar entries read, often repeatedly. Further channels include clustering activations, true prefix parameters and prefix evaluations, acceptance and scaling parameters, and walk orientation rows. A subtotal containing label requests and the D1 and D2 channels excludes those additional accesses and cannot represent the total victim-access cost. The distillation baseline is also assisted in this sweep, starting from victim biases and the victim output layer. A comparison with independent label-only distillation requires a baseline without those initial parameters. A cost comparison must specify whether it counts batched calls, input-level requests, distinct inputs, or parameter entries.

\section{Conclusion}
\label{sec:conclusion}
We combined assisted weight-direction extraction with bias estimation, sign search, output-layer fitting, and gradient-masked refinement. We proved conditional bounds for direction estimation with approximate prefixes, preservation of local hitting-distance orderings, prediction changes from aligned parameter errors, and sign selection with independent scoring data. We also specified a projected update that enforces exact row preservation. The seeded experiments show substantially more accepted directions under Leaky ReLU on the largest synthetic architecture, but no uniform deepest-layer advantage on the smallest. The reconstruction procedure produces a complete classifier, but the privileged accesses used by extraction prevent these experiments from establishing an end-to-end attack using labels alone.

The synthetic-data comparisons do not show a uniform advantage over distillation. The two CIFAR comparisons favor assisted reconstruction on all three reported functional metrics, although agreement remains below $60\%$. These results support studying reconstruction from partial directions. They do not establish exact parameter recovery or uniformly reliable sign recovery.

A strict-oracle evaluation must replace the true prefixes and ground-truth acceptance checks, account for queries per input, and reserve a final test slice that is not used for early stopping. The gradient-masked update changes row magnitudes. The restoration variant preserves the accepted rows at every recorded update and checked checkpoint, consistent with Eq.~\eqref{eq:row-projection}. Including the entering checkpoint prevents deterioration on the stopping sample, while the seeded results show that held-out agreement may still decrease. The victim-free walk replacement removes the operational victim-weight access in the guard tests, but gives less accurate signs and still receives assisted extraction artifacts. High agreement after reconstruction therefore does not resolve the sign-recovery problem. The local hitting-distance stability bound identifies the perturbation and gap conditions needed to preserve a prescribed distance vote, without assuming that the vote is correct. Further work can compare the label-trained output layer with recovery from decision-boundary constraints~\cite{canalesmartinez2025somelayers} and determine when the local rank conditions suffice for reliable direction estimation.

\backmatter

\bmhead{Data availability} The experimental data will be made available at~\cite{experimentaldata}. The implementation and experimental scripts are available at~\cite{ourimp}.

\appendix
\section{Mathematical Results and Implementation}
\label{app:provenance}
We give the correspondence between the mathematical results and the implementation~\cite{ourimp}. The propositions state sufficient conditions for their conclusions. The numerical experiments evaluate the implemented procedures under the configurations in Section~\ref{sec:exp}.

The signature-recovery routines \texttt{intersect}, \protect\path{is_consistent}, and \protect\path{extract_weights} in \protect\path{signature_recovery/recover_weights.py} implement the tangent-space calculations of Proposition~\ref{prop:dual-recovery}. The routines \protect\path{forward_around} and \protect\path{relu_around} transport those constraints through the prefix. Proposition~\ref{prop:approximate-prefix} bounds the effect of errors in that transport. Section~\ref{sec:leaky:algebra} specifies the Leaky ReLU slopes and the dimensions of the local maps.

The sign walk in \protect\path{sign_recovery.py} compares distances on the two sides of a critical hyperplane. Equation~\eqref{eq:sign-norm} describes the local velocity asymmetry under matched controls. Proposition~\ref{prop:local-vote-stability} gives conditions under which an approximate affine calculation preserves a distance ordering. Lemma~\ref{lem:sign-vote} bounds majority-vote error when the individual votes are independent and have a positive correctness advantage. These are separate conditions. Stability of a distance ordering does not imply correctness of its sign interpretation.

The routine \protect\path{recover_biases_from_duals} in \protect\path{bias_recovery.py} uses Eq.~\eqref{eq:bias-eq}. The routines \protect\path{_score_objective} and \protect\path{_flip_neuron} in \protect\path{sign_search.py} evaluate candidates and jointly change a row and its bias. Proposition~\ref{prop:sign-selection} applies when the candidate family is fixed independently of the scoring sample. The routine \protect\path{recover_output_layer} in \protect\path{output_layer_recovery.py} fits the decoder to labels. Proposition~\ref{prop:network-error} bounds prediction changes from aligned parameter errors.

Proposition~\ref{prop:row-preservation} describes the masked AdamW update in \texttt{refinement.py}. Equation~\eqref{eq:row-projection} specifies exact row restoration. The checkpoint-selection rule in Section~\ref{sec:mlrecon:frozen} retains the entering model as a candidate and compares checkpoints on a fixed stopping sample. Its deterministic guarantee concerns that sample, while held-out agreement is evaluated separately.

\end{document}